\documentclass[sn-apa]{sn-jnl}

\usepackage{graphicx}%
\usepackage{multirow}%
\usepackage{amsmath,amssymb,amsfonts}%
\usepackage{amsthm}%
\usepackage{mathrsfs}%
\usepackage[title]{appendix}%
\usepackage{xcolor}%
\usepackage{textcomp}%
\usepackage{manyfoot}%
\usepackage{booktabs}%
\usepackage{algorithm}%
\usepackage{algorithmicx}%
\usepackage{algpseudocode}%
\usepackage{listings}%
\usepackage{subfigure}
\usepackage{natbib}

\newtheorem{theorem}{Theorem}
\usepackage{color, enumerate}
\newtheorem{definition}{Definition}%
\newtheorem{assumption}{Assumption}
\newtheorem{lemma}{Lemma}
\newtheorem{corollary}{Corollary}
\begin{document}

\title[TRBoost: A Generic Gradient Boosting Machine based on Trust-region Method]{TRBoost: A Generic Gradient Boosting Machine based on Trust-region Method}

\author[1]{\fnm{Jiaqi} \sur{Luo}}\email{jiaqi.luo@dukekunshan.edu.cn}

\author[1]{\fnm{Zihao} \sur{Wei}}\email{zihao.wei@dukekunshan.edu.cn}
\author[1]{\fnm{Junkai} \sur{Man}}\email{junkai.man@dukekunshan.edu.cn}

\author*[1]{\fnm{Shixin} \sur{Xu}}\email{shixin.xu@dukekunshan.edu.cn}

\affil*[1]{\orgdiv{Data Science Research Center}, \orgname{Duke Kunshan University}, \orgaddress{\street{No.8 Duke Ave}, \city{Kunshan}, \postcode{215300}, \state{Jiangsu Province}, \country{China}}}


\abstract{Gradient Boosting Machines (GBMs) have demonstrated remarkable success in solving diverse problems by utilizing Taylor expansions in functional space. However, achieving a balance between performance and generality has posed a challenge for GBMs. In particular, gradient descent-based GBMs employ the first-order Taylor expansion to ensure applicability to all loss functions, while Newton's method-based GBMs use positive Hessian information to achieve superior performance at the expense of generality.
To address this issue, this study proposes a new generic Gradient Boosting Machine called Trust-region Boosting (TRBoost). In each iteration, TRBoost uses a constrained quadratic model to approximate the objective and applies the Trust-region algorithm to solve it and obtain a new learner. Unlike Newton's method-based GBMs, TRBoost does not require the Hessian to be positive definite, thereby allowing it to be applied to arbitrary loss functions while still maintaining competitive performance similar to second-order algorithms.
The convergence analysis and numerical experiments conducted in this study confirm that TRBoost is as general as first-order GBMs and yields competitive results compared to second-order GBMs. Overall, TRBoost is a promising approach that balances performance and generality, making it a valuable addition to the toolkit of machine learning practitioners.}




\maketitle

\section{Introduction}
Gradient Boosting Machines (GBMs) \citep{friedman2000additive,friedman2001greedy,friedman2002stochastic,natekin2013gradient} are popular ensemble models that have achieved state-of-the-art results in many data science tasks. From an optimization perspective, GBMs approximate the optimal model using the line search method in the hypothesis space. This involves computing a search direction and finding a suitable step size that satisfies certain conditions to get the new learner. Based on the degree of Taylor expansion, GBMs can be classified into two categories: first-order GBMs and second-order GBMs. First-order GBMs use only gradient information to generate new learners, while second-order GBMs use Newton's direction. Second-order methods generally outperform first-order algorithms \citep{sigrist2021gradient}, but the Hessian of loss must be positive. In contrast, first-order algorithms have no restrictions on objective functions.

Note that the Taylor expansion is only a local approximation of the given function, so we can limit the variables to a small range in which the approximation function is trustworthy.
When the feasible region is a compact set, according to the extreme value theorem, the optimal solution can be achieved, whether it is convex or not.
This provides one way to break through the dilemma of current GBMs by adding some constraints on the region where Taylor expansion is used. 
The idea of adding constraints is exactly the key concept of the Trust-region method.
Trust-region method \citep{yuan2000review} defines a region around the current iterate and applies a quadratic model to approximate the objective function in this region. 
Benefiting from the constraint, Trust-region methods do not require a quadratic coefficient positive definite.
In general, Trust-region methods are more complex to solve than line search methods. However, since the loss functions are usually convex and one-dimensional, Trust-region methods can also be solved efficiently.

This paper presents TRBoost, a generic gradient boosting machine based on the Trust-region method. We formulate the generation of the learner as an optimization problem in the functional space and solve it using the Trust-region method. TRBoost benefits from Trust-region's ability to handle arbitrary differentiable losses without requiring a positive Hessian. Moreover, the adaptive radius mechanism allows the boosting algorithm to modify the target values and the number of learners automatically.
Theoretical analysis shows that TRBoost has the same convergence rate as previous algorithms. Empirical results demonstrate that Trust-region's strength can enhance performance competitively compared to other models, regardless of the losses and learners. 
Furthermore, We also discuss the role of Hessian and present experiments that show the degree of expansion has minimal impact on our method.

The main contributions of this paper are summarized as follows:
\begin{itemize}
    \item  A gradient boosting machine that works with any learners and loss functions is proposed. It can adaptively adjust the target values and evaluate the new learner in each iteration. The algorithm maintains a balance between performance and generality. It is as efficient as Newton's method than the first-order algorithm when the loss is strictly convex. And when the Hessian is not positive definite, similar results as the first-order method can be obtained.
    \item We prove that when the loss is MAE, TRBoost has $O(\frac{1}{T})$ rate of convergence; when the loss is MSE, TRBoost has a quadratic rate and a linear rate is achieved when the loss is Logistic loss.
    \item Some theoretical analysis and experiments prove that gradient mainly determines the results and Hessian plays a role in further improvement, for both line search-based GBMs and TRBoost.
\end{itemize}

\section{Backgrounds}
\subsection{Gradient Boosting and Related Work}
Mathematically, for a given feature set $X$ and a label set $Y$, supervised learning is to find an optimal function $F^*$ in a hypothesis space $\mathcal{F} = \{F|Y=F(X)\}$ by minimizing the objective function $\mathcal{L}$ on a given training set $\mathcal{D}=\{(\mathbf{x}_1, y_1), (\mathbf{x}_2, y_2), \cdots, (\mathbf{x}_n, y_n)\}$~ ($\mathbf{x}_i\in X$, $y_i\in Y$):
\begin{equation}
\label{e.empirical}
    \min_{F\in \mathcal{F}} \mathcal{L}(F) =   \frac{1}{n}\sum_{i=1}^{n}l(y_i, F(\mathbf{x}_i)).
\end{equation}

Boosting finds $F^*$ in a stagewise way by sequentially adding a new learner $f_t$ to the current estimator $F_{t-1}$, i.e., $F_t = F_{t-1}+f_t$. Gradient Boosting is inspired by numerical optimization and is one of the most successful boosting algorithms. It originates from the work of Freund and Schapire \citep{freund1997decision} and is later developed by Friedman \citep{friedman2000additive,friedman2001greedy}. 

Since GBMs can be treated as functional gradient-based techniques,   different approaches in optimization can be applied to construct new boosting algorithms. For instance, gradient descent and Newton's method are adopted to design first-order and second-order GBMs \citep{friedman2000additive,friedman2001greedy,chen2016xgboost}.
Nesterov's accelerated descent and stochastic gradient descent induce the Accelerated Gradient Boosting \citep{lu2020accelerating,biau2019accelerated} and Stochastic Gradient Boosting \citep{friedman2002stochastic} respectively. Also, losses in probability space can generate new methods that can deal with the probabilistic regression \citep{duan2020ngboost}.

Linear functions, splines, and tree models are three commonly used base learners \citep{friedman2001greedy, buhlmann2003boosting, buehlmann2006boosting, schmid2008boosting}. Among them, the decision tree is the  first choice and most of the popular optimizations for learners are tree-based. XGBoost \citep{chen2016xgboost} presents a fantastic parallel tree learning method that can enable the Gradient Boosting Decision Tree (GBDT) to handle large-scale data. Later, LightGBM \citep{ke2017lightgbm} and CatBoost \citep{prokhorenkova2018catboost} propose several novel algorithms to further improve the efficiency and performance of GBDT in processing big data.

Numerous scholars have confirmed the convergence of gradient boosting. Bickel et al . \citep{bickel2006some} initially demonstrated a sub-linear rate when the loss function is logistic, while Telgarsky \citep{telgarsky2012primal} later established the linear convergence for logistic loss. Additionally, some researchers \citep{grubb2011generalized} have demonstrated a linear convergence rate for boosting with strongly convex losses.

\subsection{Trust-region Method}
Suppose the unconstrained optimization is 
\begin{equation}
    \min_{\mathbf{x}\in \mathbb{R}^n} f(\mathbf{x}),
\end{equation}
where $f$ is the \textit{objective function} and $\mathbf{x}$ is the \textit{decision variables}.
We employ a quadratic model, 
\begin{equation}
\label{e.quadratic}
    f(\mathbf{x}_k+\mathbf{p}) = f(\mathbf{x}_k)+ \bigtriangledown f(\mathbf{x}_k)^T \mathbf{p} + \frac{1}{2}\mathbf{p}^T \mathbf{B}_k \mathbf{p}, 
\end{equation}
to characterize the properties of $f$ around $\mathbf{x}_k$, where $\mathbf{B}_k$ is a symmetric matrix. It is an approximation of the Hessian matrix and does not need to be positive definite. 
At each iteration, Trust-region solves the following subproblem to obtain each step
\begin{gather}
\label{e.subproblem}
\min_{\mathbf{p}\in \mathbb{R}^n}  ~~m_k(\mathbf{p}) = f(\mathbf{x}_k)+ \bigtriangledown f(\mathbf{x}_k)^T \mathbf{p} + \frac{1}{2}\mathbf{p}^T \mathbf{B}_k \mathbf{p}\\
\textrm{s.t.}~~  \|\mathbf{p}\| \le \bigtriangleup_k, \nonumber
\end{gather}
where $\bigtriangleup_k$ is the trust region radius. The strategy for choosing $\bigtriangleup_k$ is critical since the radius reflects the confidence in $m_k(\mathbf{p})$ and determines the convergence of the algorithm. $\|\cdot\|$ is a norm that is usually defined to be the Euclidean norm.
The following theorem provides the condition for problem \eqref{e.subproblem} to process a global solution.

\begin{theorem}
\label{thm.1}
The vector $\mathbf{p}^*$ is a global solution of \eqref{e.subproblem} if and only if $\mathbf{p}^*$ is feasible and there is a scalar $\lambda_k \ge 0$ such that the following conditions are satisfied:
\begin{align*}
    &(\mathbf{B_k}+\lambda_k \mathbf{I_k})\mathbf{p}^* = -\bigtriangledown f(\mathbf{x}_k),\\
    &\lambda_k(\bigtriangleup_k-\|\mathbf{p}^*\|) = 0,\\
    &(\mathbf{B_k}+\lambda_k \mathbf{I_k})~\mbox{is~positive~semi-definite.}
\end{align*}
\end{theorem}

Given a step $\mathbf{p}_k$, we define the ratio between  the \textit{actual reduction} and the \textit{predicted reduction}
\begin{equation}
\label{e.trratio}
    \rho_k = \frac{f(\mathbf{x}_k)-f(\mathbf{x}_k+\mathbf{p}_k)}{m_k(\mathbf{0})-m_k(\mathbf{p}_k)},
\end{equation}
 which measures how well $m_k(\mathbf{p})$ approximates and helps to  adjust the radius $\bigtriangleup_k$. The following Algorithm \ref{alg.trust} describes the standard
 Trust-region method, more details and the proof of Theorem \ref{thm.1} can be found in \citep{nocedal1999numerical}.

\begin{algorithm}
\footnotesize
\caption{Trust-region algorithm}
\label{alg.trust}
\begin{algorithmic}[1]
    \State Given max radius $\bigtriangleup_{max}$, initial radius $\bigtriangleup_0\in (0, \bigtriangleup_{max})$, initial point $x_0$, $k\leftarrow 0$;
    \State Given parameters $0\le \eta<\hat{\rho}_1 < \hat{\rho}_2 < 1$, $\gamma_1 < 1 < \gamma_2$;
    \While{Stop condition not met}
    \State Obtain $\mathbf{p}_k$ by solving \eqref{e.subproblem};
    \State Evaluate $\rho_k$ from \eqref{e.trratio};
    \If{$\rho_k < \hat{\rho}_1$}
        \State $\bigtriangleup_{k+1}=\gamma_1\bigtriangleup_k$;
    \Else
        \If{$\rho_k > \hat{\rho}_2$ and $\|\mathbf{p}_k\|=\bigtriangleup_k$}
            \State $\bigtriangleup_{k+1}=\min(\gamma_2\bigtriangleup_k, \bigtriangleup_{max})$;
        \Else
            \State $\bigtriangleup_{k+1}=\bigtriangleup_k$;
        \EndIf
    \EndIf
    \If{$\rho_k > \eta$}
      \State $\mathbf{x}_{k+1}=\mathbf{x}_k+\mathbf{p}_k$;
    \Else
      \State $\mathbf{x}_{k+1}=\mathbf{x}_k$;
    \EndIf
    \EndWhile
\end{algorithmic}
\end{algorithm}

In practice, $\hat{\rho}_1=0.25$, $\hat{\rho}_2=0.75$, $\gamma_1=0.25$, $\gamma_2=2$ are taken as default.

\section{Trust-region Gradient Boosting}
In this section, we first introduce the generic formulation of the algorithm TRBoost and then present the special case when base learners are decision trees.
\subsection{General Formulation}
\textbf{Objective.} Given a dataset with $n$ samples and $m$ features $\mathcal{D}=\{(\mathbf{x}_1, y_1), (\mathbf{x}_2, y_2), \cdots, (\mathbf{x}_n, y_n)\}$~($\mathbf{x}_i \in \mathbb{R}^m, y_i\in \mathbb{R}$), we define the objective of iteration $t$ as
\begin{equation}
\label{e.objfun}
    \mathcal{L}^{t} = \frac{1}{n}\sum_{i=1}^{n}l(y_i, \hat{y}^{t-1}_i+f_t(\mathbf{x}_i)),
\end{equation}
where $t \in [1, T]$ is the learner number, $\hat{y}^{t-1}_i$ is the prediction from the previous $t-1$ learners, $f_t$ is a new weak learner and $l$ is the loss function that does not need to be strictly convex. 

\noindent\textbf{Optimization and Target Values.}
Trust-region method is employed to optimize the objective, which has the following formulation
\begin{gather}
\label{e.objective}
    \min_{f_t} \mathcal{\widetilde{L}}^{t} = \frac{1}{n}\sum_{i=1}^{n}(g^{t-1}_i f_t(\mathbf{x}_i)+\frac{1}{2}b^{t-1}_i f_t(\mathbf{x}_i)^2),\\
    \textrm{s.t.}~~\vert f_t(\mathbf{x}_i)\vert\le r_i^{t}, i=1,2,\cdots, n, \nonumber
\end{gather}
where $g^{t-1}_i = \frac{\partial l(y_i,\hat{y}^{t-1}_i)}{\partial \hat{y}^{t-1}_i}$,
$b^{t-1}_i$ is the second derivative $\frac{\partial^2
l(y_i,\hat{y}^{t-1}_i)}{\partial (\hat{y}^{t-1}_i)^2}$ or some approximation.

Since each item $l(y_i, \hat{y}^{t-1}_i+f_t(\mathbf{x}_i))$ is independent, we can split \eqref{e.objective} into following $n$ one-dimensional constrained optimization problems,
\begin{gather}
\label{e.trbsubproblem}
\min_{z^{t}_i} g^{t-1}_i z^{t}_i+\frac{1}{2}b^{t-1}_i (z^{t}_i)^2\\
\textrm{s.t.}~~\vert z^{t}_i\vert\le r_i^{t}, \nonumber
\end{gather} 
where the solution $z^{t}_i$ is the target that $f_t$ needs to fit.
They are easy to solve and each $z^{t}_i$ has analytical expression, which is
\begin{equation}
\label{e.trbtarget1}
z^{t}_i = \left\{
\begin{array}{l}
sgn(-g^{t-1}_i)\times\min(\vert \frac{g^{t-1}_i}{b^{t-1}_i}\vert, r_i^{t}),~b^{t-1}_i>0; \\
sgn(-g^{t-1}_i)\times r_i^{t},~b^{t-1}_i\le0 .\\
\end{array}
\right.
\end{equation}

The segmented form \eqref{e.trbtarget1} can be rewritten in the following format
\begin{equation}
\label{e.trbtarget2}
z^{t}_i = \frac{-g^{t-1}_i}{b^{t-1}_i+\mu^{t}_i},
\end{equation}
which matches Theorem \ref{thm.1}. 
Here $\mu^{t}_i$ is a non-negative scalar related to $r_i^{t}$ ensuring the positive denominator $b^{t-1}_i+\mu^{t}_i$. The relationship between $\mu^{t}_i$ and $r^{t}_i$ is as follows:
\begin{enumerate}[(a)]
    \item If $\vert \frac{g^{t-1}_i}{b^{t-1}_i}\vert \ge r_i^{t}$, then $z_i^{t} =sgn(-g^{t-1}_i)\vert\frac{g^{t-1}_i}{b^{t-1}_i}\vert=\frac{-g^{t-1}_i}{b^{t-1}_i}$. Let $\mu_i^{t}=0$, then $z^{t}_i = \frac{-g^{t-1}_i}{b^{t-1}_i+\mu^{t}_i}$.
    \item  If $\vert\frac{g^{t-1}_i}{b^{t-1}_i}\vert< r_i^{t}$ or $b^{t-1}_i\le 0$, then $z_i^{t} =sgn(-g^{t-1}_i){r^{t}_i}$. Let $\mu_i^{t}=\frac{\vert g^{t-1}_i \vert}{r_i^{t}}-b^{t-1}_i$, then $z^{t}_i = \frac{-g^{t-1}_i}{b^{t-1}_i+\mu^{t}_i}$.
\end{enumerate} 

As the number of instances increases, there will be a mass of $\mu^{t}_i$ that needs to be adjusted in each iteration, which will be time-consuming.
Therefore, we replace different $r^{t}_i$ with single value  $\mu^{t} = \max\{\mu^{t}_i\}_{i=1}^{n}$ for simplicity and obtain the target value
\begin{equation}
\label{e.trbvalue}
    z^{t}_i = \frac{-g^{t-1}_i}{b^{t-1}_i+\mu^{t}}.
\end{equation}

\noindent\textbf{Update Strategy.}
Let $\hat{z}^{t}_i=f_t(\mathbf{x_i})$ be the output of the new learner, we provide two ratios to determine how to update $\mu^{t}$. The first one is the same as \eqref{e.trratio}, which is
\begin{equation}
\label{e.trbratio1}
    \rho_1^{t} = \frac{\mathcal{L}^{t-1}-\mathcal{L}^{t}}{-\frac{1}{n}\sum_{i=1}^{n}[g^{t-1}_i \hat{z}^{t}_i+\frac{1}{2}b^{t-1}_i (\hat{z}^{t}_i)^2]}.
\end{equation}
When $b_i^{t}=\frac{\partial^2
l(y_i,\hat{y}^{t-1}_i)}{\partial (\hat{y}^{t-1}_i)^2}$ and the loss function is $L_2$ loss, $\rho_1^{t}$ is always 1, so that $\mu^{t}$ can not be updated. Therefore, the second ratio is introduced based on the difference of $\mathcal{L}$ to overcome the drawbacks, which is
\begin{equation}
\label{e.trbratio2}
    \rho_2^{t} = \frac{\mathcal{L}^{t-1}-\mathcal{L}^{t}}{\frac{1}{n}\sum_{i=1}^{n}\vert \hat{z}^{t}_i\vert}.
\end{equation}

\noindent\textbf{Algorithm.} 
The detailed procedure of generic TRBoost is presented in Algorithm \ref{alg.trgbm}, which is on the basis of Algorithm\ref{alg.trust}. $\epsilon_1$, $\epsilon_2$, and $\gamma$ are three constants that control the update of $\mu^t$. If $\rho^t < \epsilon_1$, it means that the radius is too large such that the Taylor approximation is not good. On the other hand, although $\rho^t > \epsilon_2$ implies a good approximation, the sharp decrease in the objective $\mathcal{L}^{t}$ may induce overfitting. Hence in both cases, we need to reduce the radius (enlarge $\mu^t$) and let $\rho^t$ be close to 1. The default values of $\epsilon_1$ and $\epsilon_2$ are 0.9 and 1.1 respectively. Besides, large $\gamma$ will cause the target values \eqref{e.trbvalue} to drop to 0 fast and make the algorithm converge early, so we set the default value to 1.01. Moreover, $\rho^t < \eta$ indicates that the corresponding weak  learner  contributes little to the decline of the objective function, thus it will not be added to the ensemble, which can reduce the model complexity and lower the risk of overfitting.

\begin{algorithm}
\footnotesize
\caption{Generic Trust-region Gradient Boosting}
\label{alg.trgbm}
\begin{algorithmic}[1]
    \State Given initial constants $\mu^0 \ge 0$, initial tree $F_0$;
    \State Given parameters $0\le \eta \leq \epsilon_1 < 1 < \epsilon_2$, $\gamma > 1$;
    \While{Stop condition not met}
    \State Calculate $z^{t}_i, i=1,2,\cdots,n$ from \eqref{e.trbvalue};
    \State Obtain $f_t$ by fitting $\{\mathbf{x}_i, z^{t}_i\}_{i=1}^{n}$;
    \State Compute $\rho^{t}$ using \eqref{e.trbratio1} or \eqref{e.trbratio2};
    \If{$\rho^{t} < \epsilon_1$ or $\rho^{t} > \epsilon_2$}
        \State $\mu^{t+1}=\gamma\mu^{t}$;
    \Else
        \State $\mu^{t+1}=\mu^{t}$;
    \EndIf
    \If{$\rho^{t} > \eta$}
        \State $F_{t+1}=F_{t}+f_t$;
    \Else
        \State $F_{t+1}=F_{t}$;
    \EndIf
    \EndWhile
\end{algorithmic}
\end{algorithm}

\subsection{Trust-region Boosting Tree}
\noindent\textbf{Leaf Values.}
For a tree with a fixed structure, it can be written as a piecewise linear function $f_t(\mathbf{x}) =\sum_{j=1}^{k}C_{j}^{t} I(\mathbf{x}\in R_{j}^{t})$, where $C_{j}^{t} \in \mathbb{R}$ is the value of leaf $j$ and $R_{j}^{t} \subseteq \mathbb{R}^m$ is the corresponding region.
Suppose the instances in $R_{j}^{t}$ is $\{\mathbf{x}_{j_1}, \mathbf{x}_{j_2}, \cdots, \mathbf{x}_{j_{\ell}}\}$, we have
\begin{equation*}
    \vert C_j^t\vert = \vert f_t(\mathbf{x}_s)\vert \le r_s^t, s=j_1,j_2,\cdots, j_{\ell}. 
\end{equation*}
Let $r_j^t = \min\{r_s^t\}_{s=j_{1}}^{j_{\ell}}$, then $\vert C_{j}^{t}\vert\le r_{j}^{t}$.
Since there is no intersection between different region $R_{j}^{t}$, $n$ subproblems in \eqref{e.trbsubproblem} are equivalent to the following $k$ models about $C_{j}^{t}$, $j=1, 2, \dots, k$:
\begin{gather}
\label{e.trbopt}
\min_{C_{j}^{t}} ~~\frac{1}{2}B_{j}^{t-1} (C_{j}^{t})^2 + G_{j}^{t-1} C_{j}^{t}\\
\textrm{s.t.}~~ \vert C_{j}^{t}\vert\le r_{j}^{t}, \nonumber
\end{gather}
where $G_{j}^{t-1} = \sum_{x_{i}\in R_{j}^{t}}g_{i}^{t-1}, B_{j}^{t-1} = \sum_{x_i\in R_{j}^{t}}b_{i}^{t-1}$.

Same as \eqref{e.trbvalue}, the optimal leaf value $C_{j}^{t}$  can be computed by
\begin{equation}
\label{e.leafvalue1}
C_{j}^{t} = \frac{-G_{j}^{t-1}}{B_{j}^{t-1}+\mu_{j}^{t}}.
\end{equation}
We use a variant of $\mu_{j}^{t}$ to reduce the influence of the number of instances on $C_{j}^{t}$, which is defined as  $\mu_{j}^{t} = \alpha^{t} n_{j}^{t}+\beta^{t}$. Here $n_{j}^{t}$ is the number of instances contained in leaf $j$, $\alpha^t$ and $\beta^t$ are two constants needed to be updated like $\mu^t$ in Algorithm \ref{alg.trgbm}, i.e., $\alpha^{t+1}=\gamma\alpha^{t}$, $\beta^{t+1}=\gamma\beta^{t}$.  Their default values are 0.1 and 10 respectively. And \eqref{e.leafvalue1} now becomes
\begin{equation}
\label{e.leafvalue2}
C_{j}^{t} = \frac{-G_{j}^{t-1}}{B_{j}^{t-1}+\alpha^{t} n_{j}^{t}+\beta^{t}}.
\end{equation}

After getting the leaf value, we can calculate the corresponding optimal value by
\begin{equation}
    \mathcal{\widetilde{L}}_{j}^{t} = \frac{1}{2}B_{j}^{t-1} (C_{j}^{t})^2 + G_{j}^{t-1} C_{j}^{t}.
\end{equation}

\noindent\textbf{Splitting Rules.}
Different functions such as \textit{gini impurity}\citep{breiman2017classification},
\textit{loss reduction} \citep{chen2016xgboost} and \textit{variance gain} \citep{ke2017lightgbm} can be applied to decide whether the leaf node should split.
In our current implementation version, we choose loss reduction as the splitting function. 

Suppose $R_P = R_L \cup R_R$, $R_L$, and $R_R$ are the corresponding regions of the left and right nodes after the parent node $R_P$ split. The loss reduction in our method is defined as follows:
\begin{equation*}
\begin{split}
    \mathcal{L}_{split} =&~   \frac{1}{2}\left[\frac{G_L^2B_L}{(B_L+\mu_L)^2}+\frac{G_R^2B_R}{(B_R+\mu_R)^2}-\frac{G_P^2B_P}{(B_P+\mu_P)^2}\right]\\
    &~+\left(\frac{-G_L^2}{B_L+\mu_L}+\frac{-G_R^2}{B_R+\mu_R}-\frac{-G_P^2}{B_P+\mu_P}\right)
\end{split}
\end{equation*}
where $G_{*}=\sum_{x_{i}\in R_{*}}g_{i}$, $B_{*}=\sum_{x_{i}\in R_{*}}b_{i}$ and $\mu_{*}=\alpha n_{*}+\beta$, $* \in \{P, L, R\}$.

\noindent\textbf{Comparison with XGBoost.}
We end this subsection by providing comparisons between TRBoost and XGBoost and show the superiority of the former one: 
\begin{enumerate}[(1)]
\item \textbf{Approximation model}: The approximation model of XGBoost at step $t$ is
\begin{equation}
    \mathcal{\widetilde{L}}_{XGB}^{t} = \sum_{i=1}^{n}(g^{t-1}_i f_t(\mathbf{x}_i)+\frac{1}{2}h^{t-1}_i f_t(\mathbf{x}_i)^2)+\Omega(f_t). \nonumber
\end{equation}
This model is unconstrained and requires $h^{t-1}_i>0$ to ensure that the minimization problem can be solved. In contrast, the bounded feasible region in TRBoost makes the approximation model able to be solved whatever the sign of the quadratic term is.

\item \textbf{Leaf value}: Although the leaf value of both methods can be written as $C = \frac{-G}{H+\mu}$, the $C$ in TRBoost is more flexible. $H$ does not need to be strictly positive and $\mu$ can be updated during iterations.

\item  \textbf{Boosting strategy}: 
In each iteration, XGBoost adds all new learners to the ensemble model, while TRBoost adds selectively.

\end{enumerate}

\section{Experiments}
\subsection{Experiments Setup}
\textbf{Datasets.} 12 datasets, which come from the UCI Machine Learning Repository \citep{Dua:2019} and OpenML \citep{vanschoren2014openml}, and 1 synthetic dataset are used for comparisons. The datasets are listed in Table \ref{T.data}.
The Noisy data is generated by the function \textit{make regression} in Scikit-Learn \citep{pedregosa2011scikit} where $10\%$ strong outliers are added to the training data. For all datasets, we randomly hold out 80$\%$ of the instances as the training set and the rest as the test set. For the training data,  20$\%$ is held out as the validation set, and  the grid search method is used to select the best parameters. After obtaining the best parameters, we retrain the model to make predictions on the test set. The entire process is repeated 5 times for all datasets.

\noindent\textbf{Loss Functions and Evaluation Metric.}
Four commonly used convex losses are employed for experiments, namely \textit{Log Loss} in classification,  \textit{Squared Error ($L_2$ Loss)}, \textit{Absolute Loss ($L_1$ Loss)} and \textit{Huber Loss} in regression. For the binary classification task, \textit{AUC} \citep{davis2006relationship} and \textit{F1-score} are chosen as the evaluation metric. For the regression task, loss functions are also applied to evaluate the performance of different methods. The $b_i^{t-1}$ in \eqref{e.trbvalue} is chosen as $\frac{\partial^2
l(y_i,\hat{y}^{t-1}_i)}{\partial (\hat{y}^{t-1}_i)^2}$.

\noindent\textbf{Implementation Details.} 
The Scikit-Learn implementation of GBDT and the XGBoost Python Package\footnote{\url{https://xgboost.readthedocs.io/en/stable/python/index.html}} are employed for comparisons, they are first-order and second-order algorithms respectively. 
Decision Tree, Linear Regression, and Cubic Spline \citep{hastie2009elements} are chosen as the base learners in TRBoost, and $\rho_1^t$ is selected as the approximation ratio.
The decision tree model is implemented by ourselves using Python\footnote{\url{https://github.com/Luojiaqimath/TRBoost}} and the other two models are implemented by Scikit-Learn.
We do not perform any special preprocessing on categorical features.
The parameters we adjust are presented in Table \ref{T.param} and others are left to their default values. All the experiments are conducted on a workstation running Ubuntu 20.04 with an Intel Core i9-10900X CPU and 128G Memory.

\begin{table}
\scriptsize
    \caption{Datasets used in experiments}
    \label{T.data}
    \centering
    \begin{tabular}{lccc}
     \hline
     Datasets & \#Ins./\#Feat. & Task/Loss & Metric\\
     \hline
     Adult              & 32561/14   &   Clf/Log   & AUC\&F1\\
     German             & 1000/20  &   Clf/Log   & AUC\&F1\\
     Electricity        & 45312/8  &   Clf/Log    & AUC\&F1\\
     Sonar              & 208/60  &   Clf/Log    & AUC\&F1\\
     Credit            & 1000/20   &   Clf/Log   & AUC\&F1\\
     Spam               & 4601/57  &   Clf/Log    & AUC\&F1\\
     \hline
     California         & 20634/8  &   Reg/$L_2$  & Loss  \\
     Concrete           & 1030/8  &   Reg/$L_2$    & Loss \\
     Energy             & 768/8  &   Reg/$L_2$    & Loss  \\
     Power              & 9568/4  &   Reg/$L_2$    & Loss  \\
     Kin8nm             & 8192/8  &   Reg/$L_2$    & Loss \\
     Wine quality       & 6497/11  &   Reg/$L_2$    & Loss  \\
     Noisy data         & 500/5  &   Reg/$L_1$\&Huber  & Loss\\
     \hline
    \end{tabular}
\end{table}

\begin{table}
\caption{Adjusted parameters}
\label{T.param}
\centering
\begin{tabular}{lcc}
\hline
Algorithm  & Parameters name  & Values \\
\hline
\multirow{2}{*}{GBDT}    & learning rate  & [0.1, 0.5, 1] \\
                         & \#estimators   & 1-100 \\
\hline
\multirow{3}{*}{XGBoost} & learning rate & [0.1, 0.5, 1] \\
                         & regularization term $\lambda$ & [0, 0.5, 1, 5] \\
                         & \#estimators   & 1-100 \\
\hline
\multirow{3}{*}{TRBoost} &  $\mu$ or $\alpha$ & [0.1, 0.5, 1, 5] \\
                         & $\eta$ &  [0, 0.01, 0.1]  \\
                         & \#estimators   & 1-100 \\
 \hline
\end{tabular}
\end{table}

\begin{figure}[ht]
\centering
\subfigure[Adult]{\includegraphics[width=0.22\textwidth]{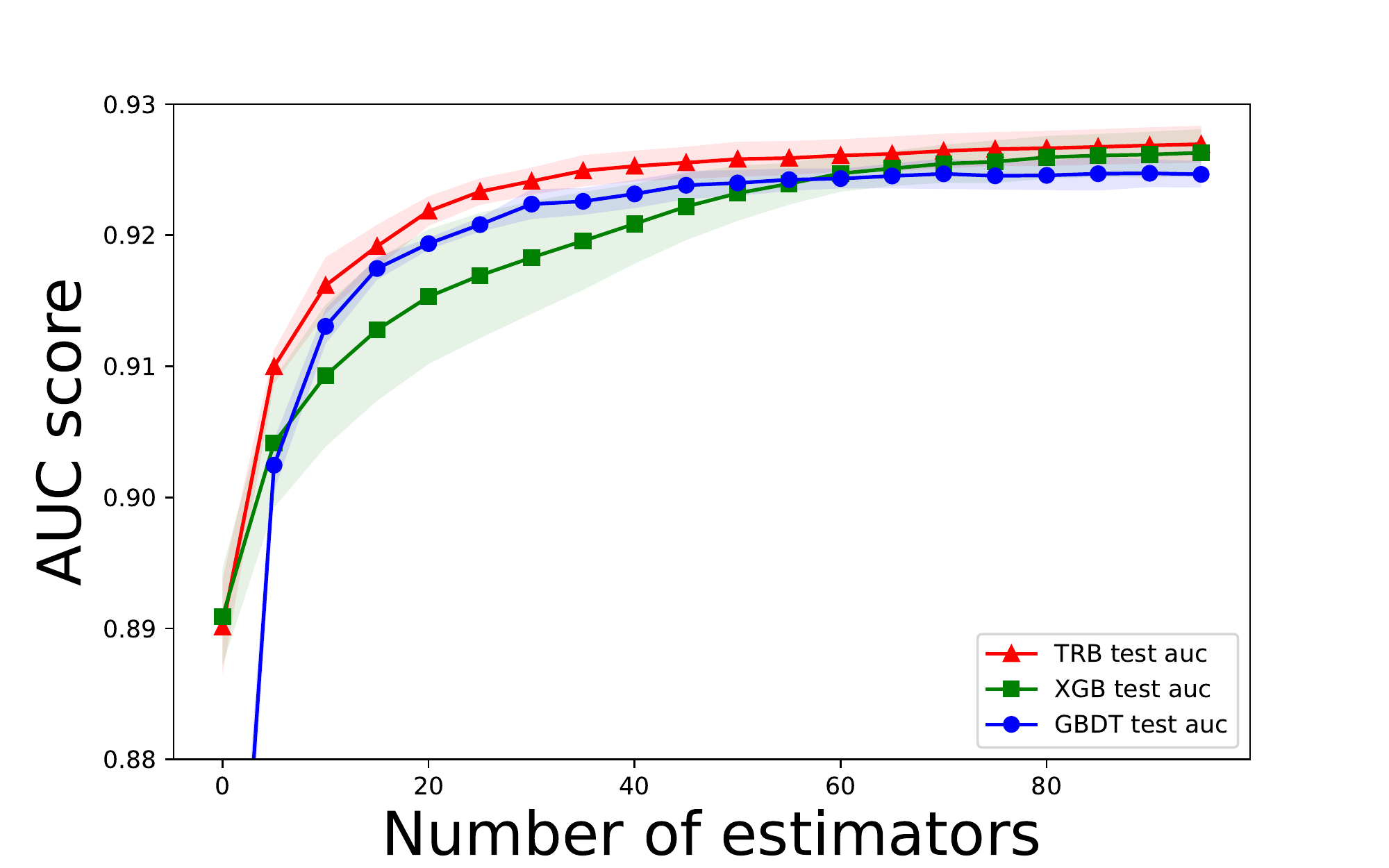}}
\subfigure[German]{\includegraphics[width=0.22\textwidth]{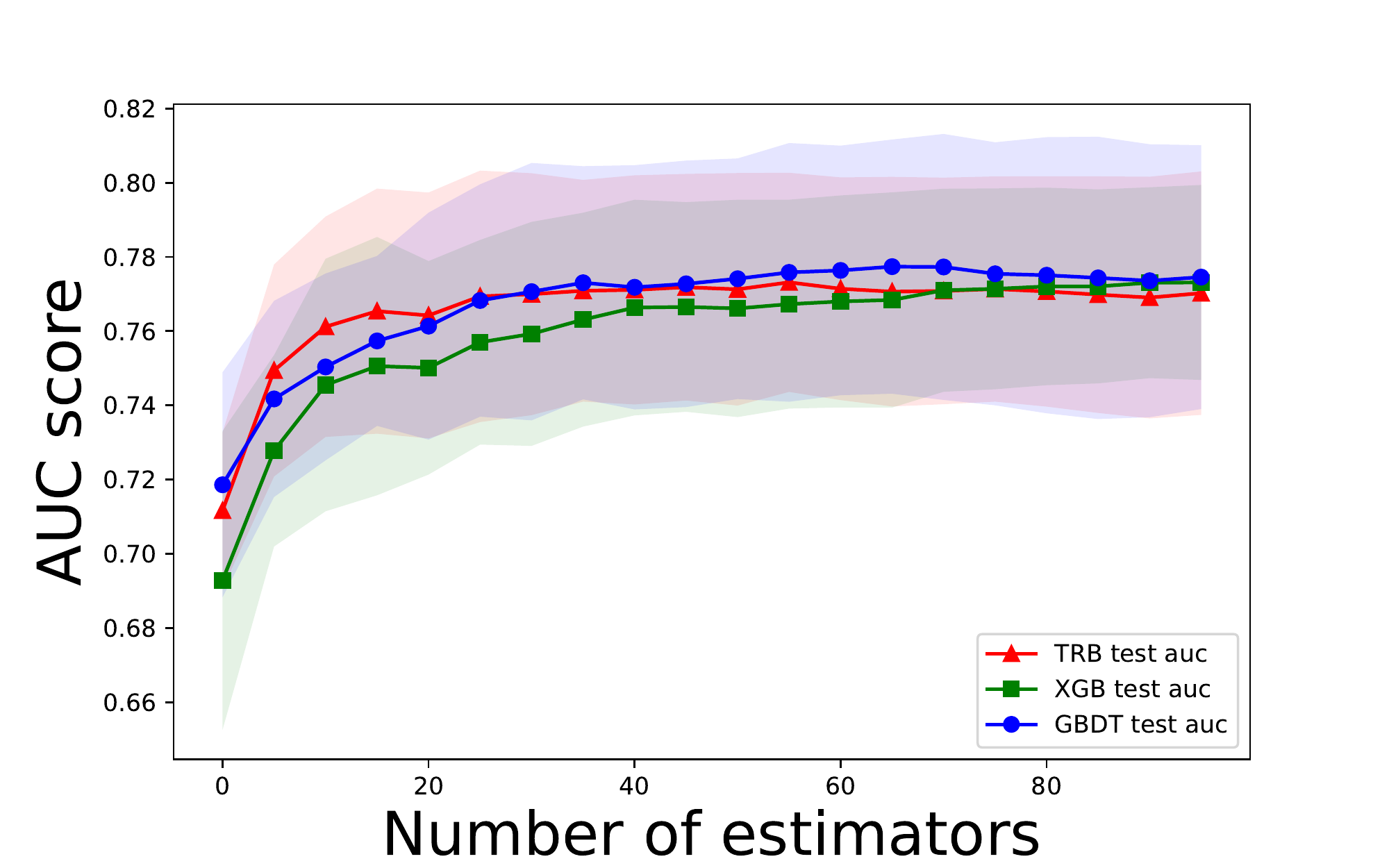}}
\subfigure[Electricity]{\includegraphics[width=0.22\textwidth]{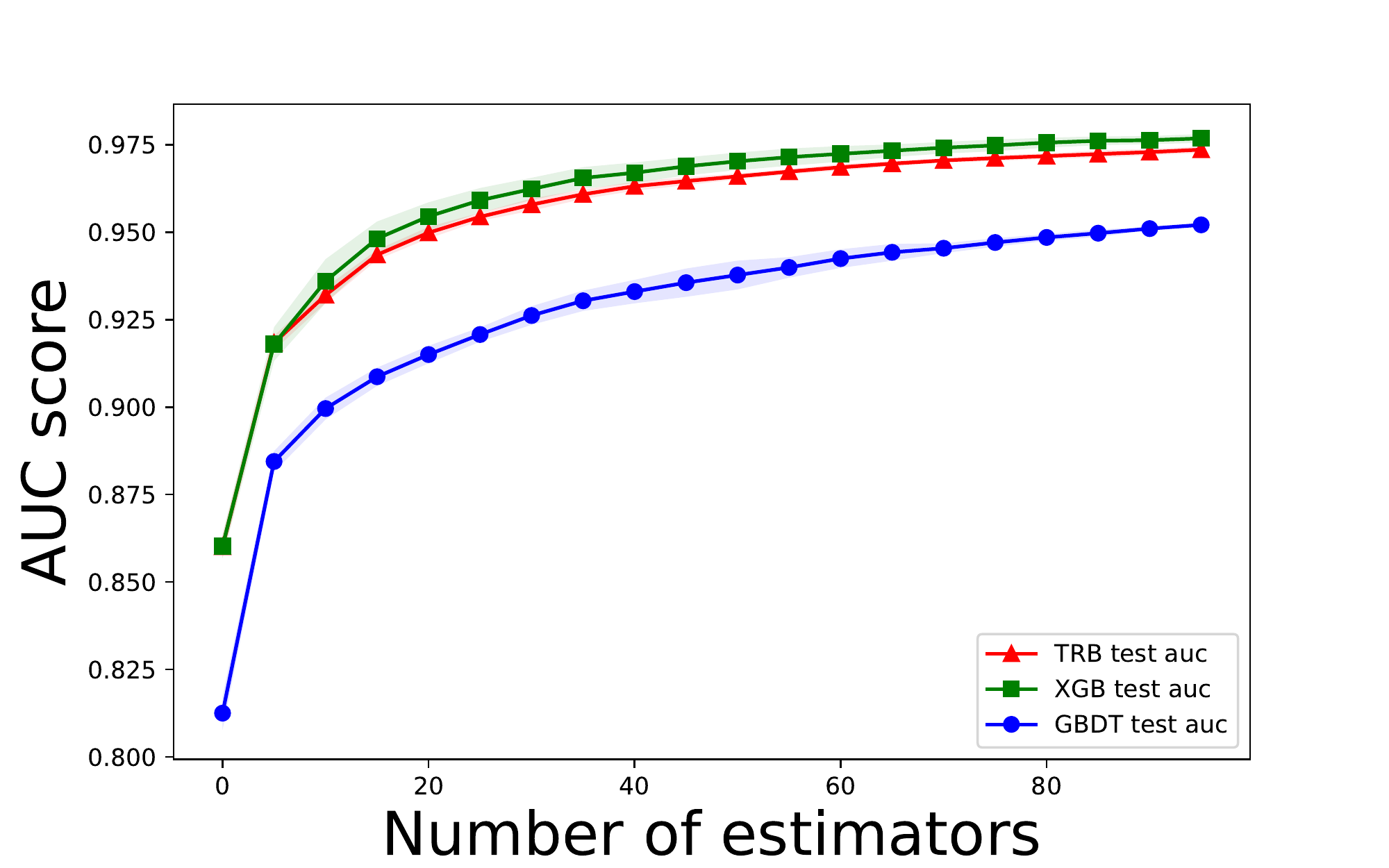}}
\subfigure[Sonar]{\includegraphics[width=0.22\textwidth]{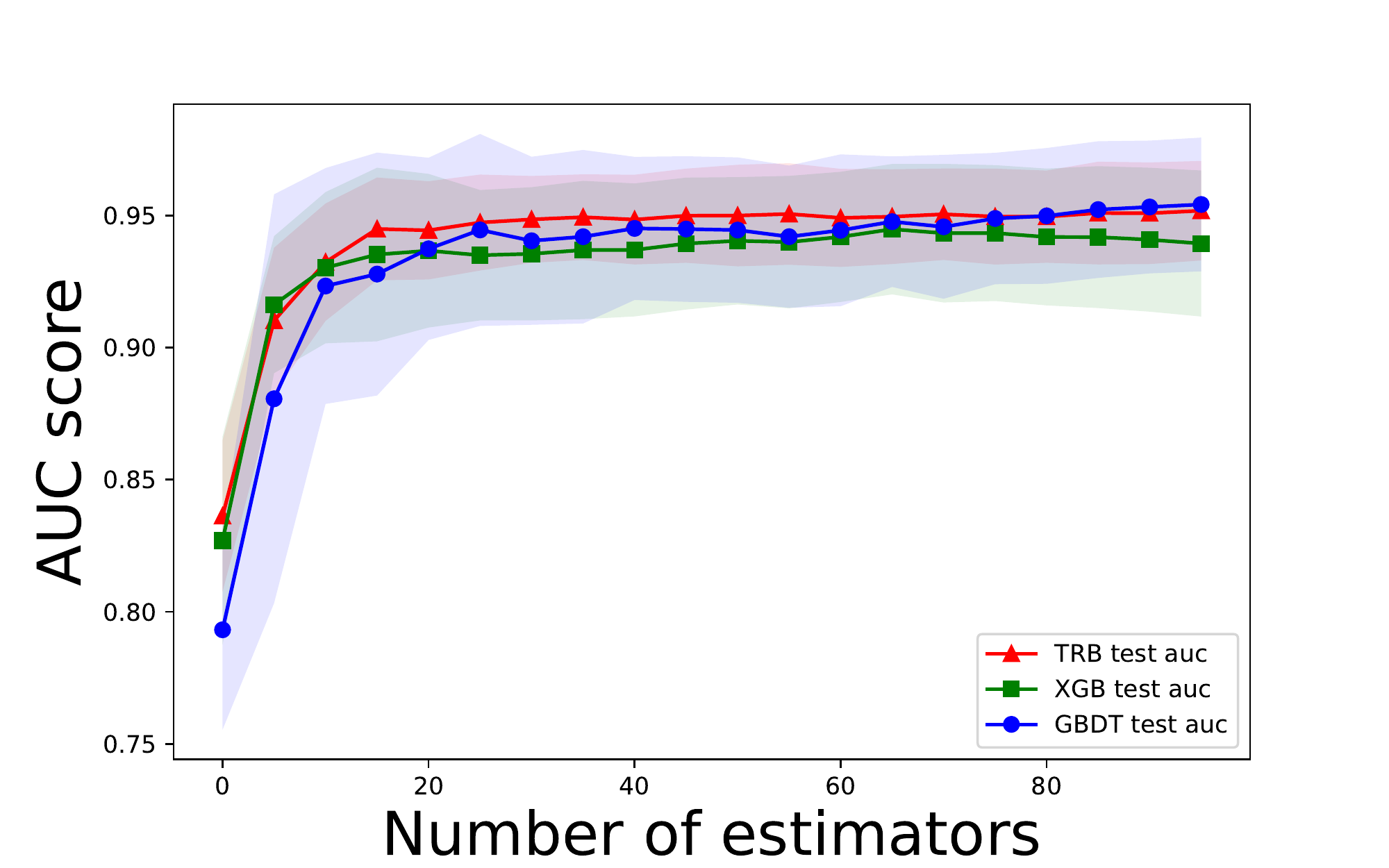}}
\subfigure[Credit]{\includegraphics[width=0.22\textwidth]{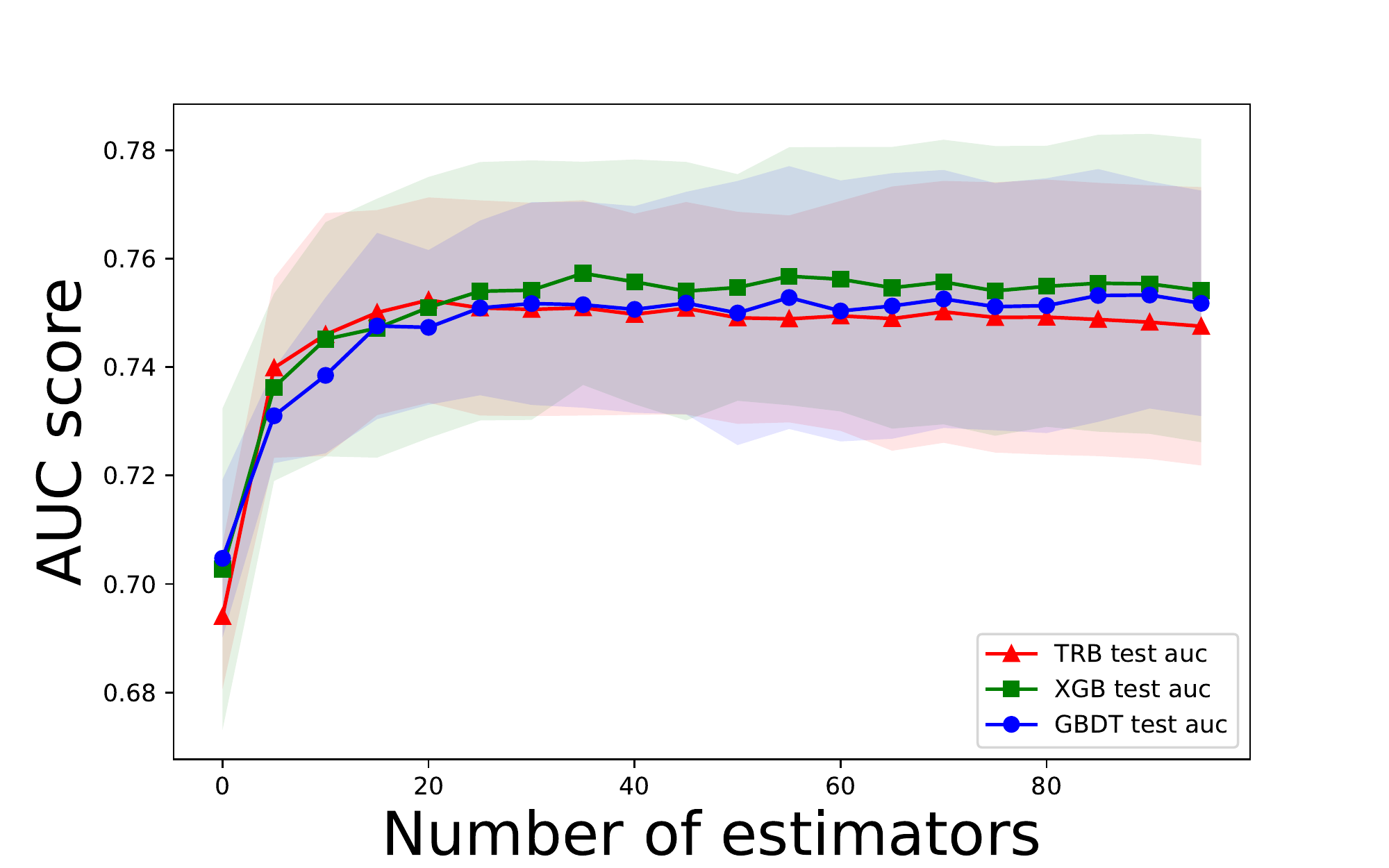}}
\subfigure[Spam]{\includegraphics[width=0.22\textwidth]{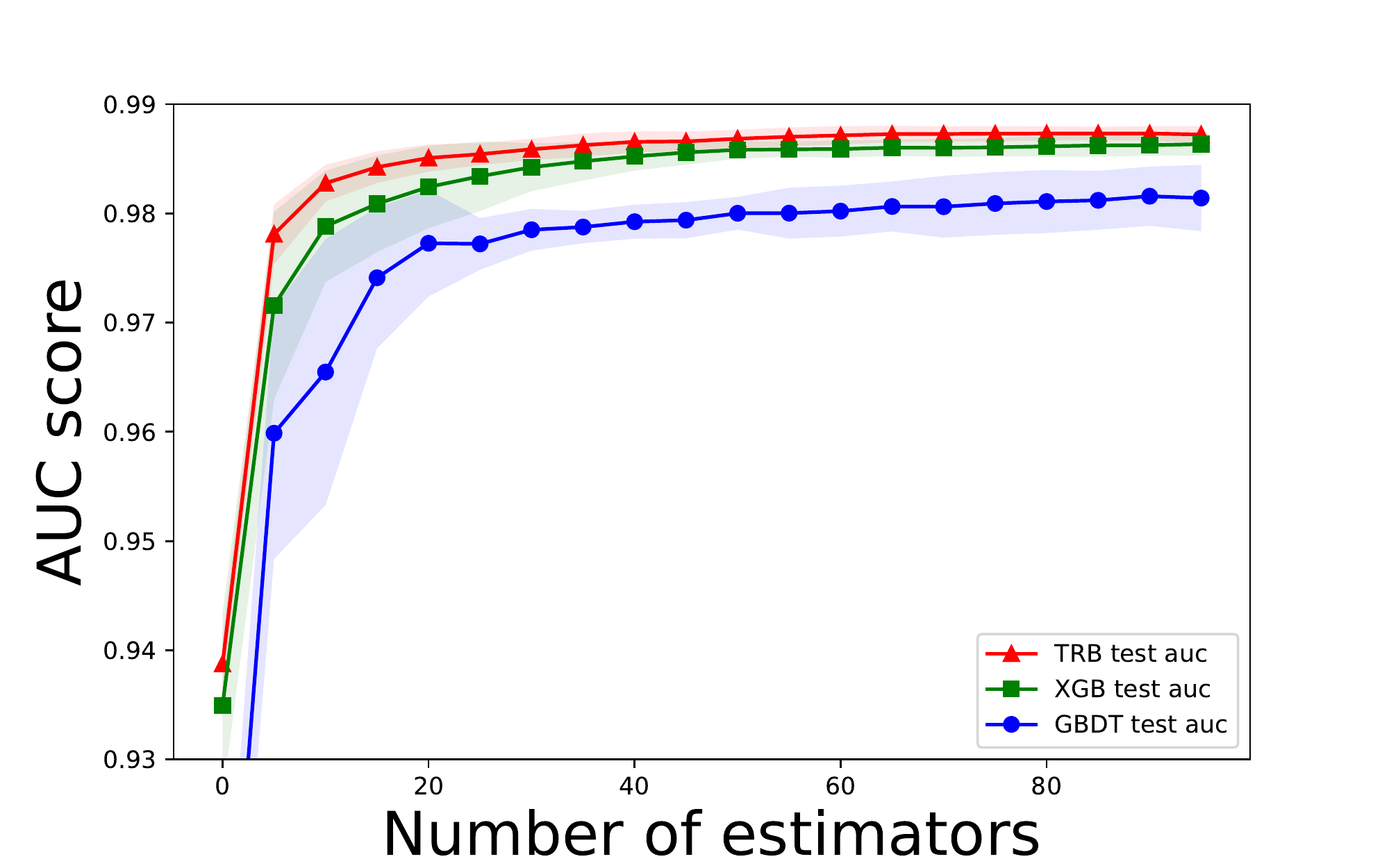}}
\subfigure[California]{\includegraphics[width=0.22\textwidth]{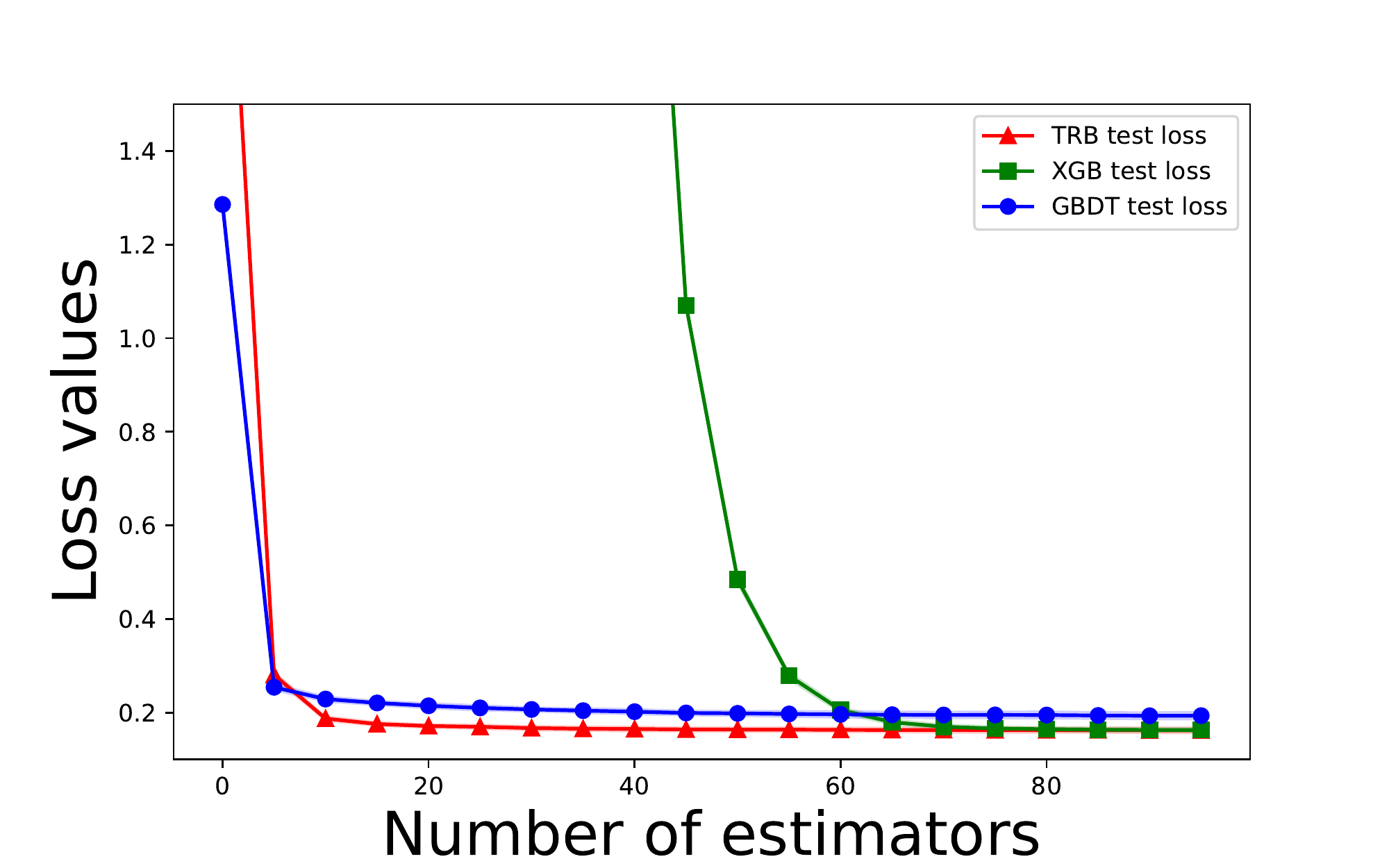}}
\subfigure[Concrete]{\includegraphics[width=0.22\textwidth]{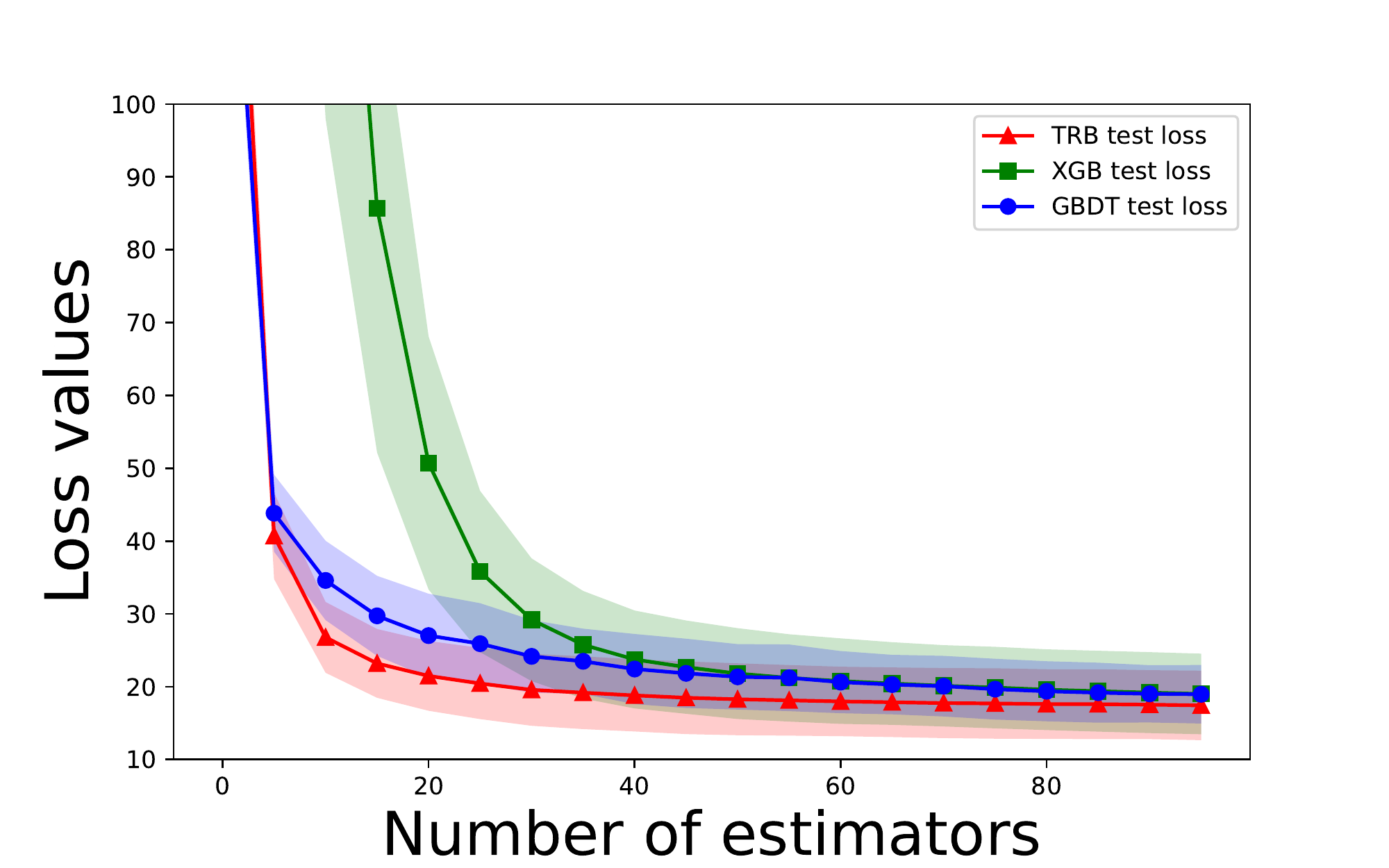}}
\subfigure[Energy]{\includegraphics[width=0.22\textwidth]{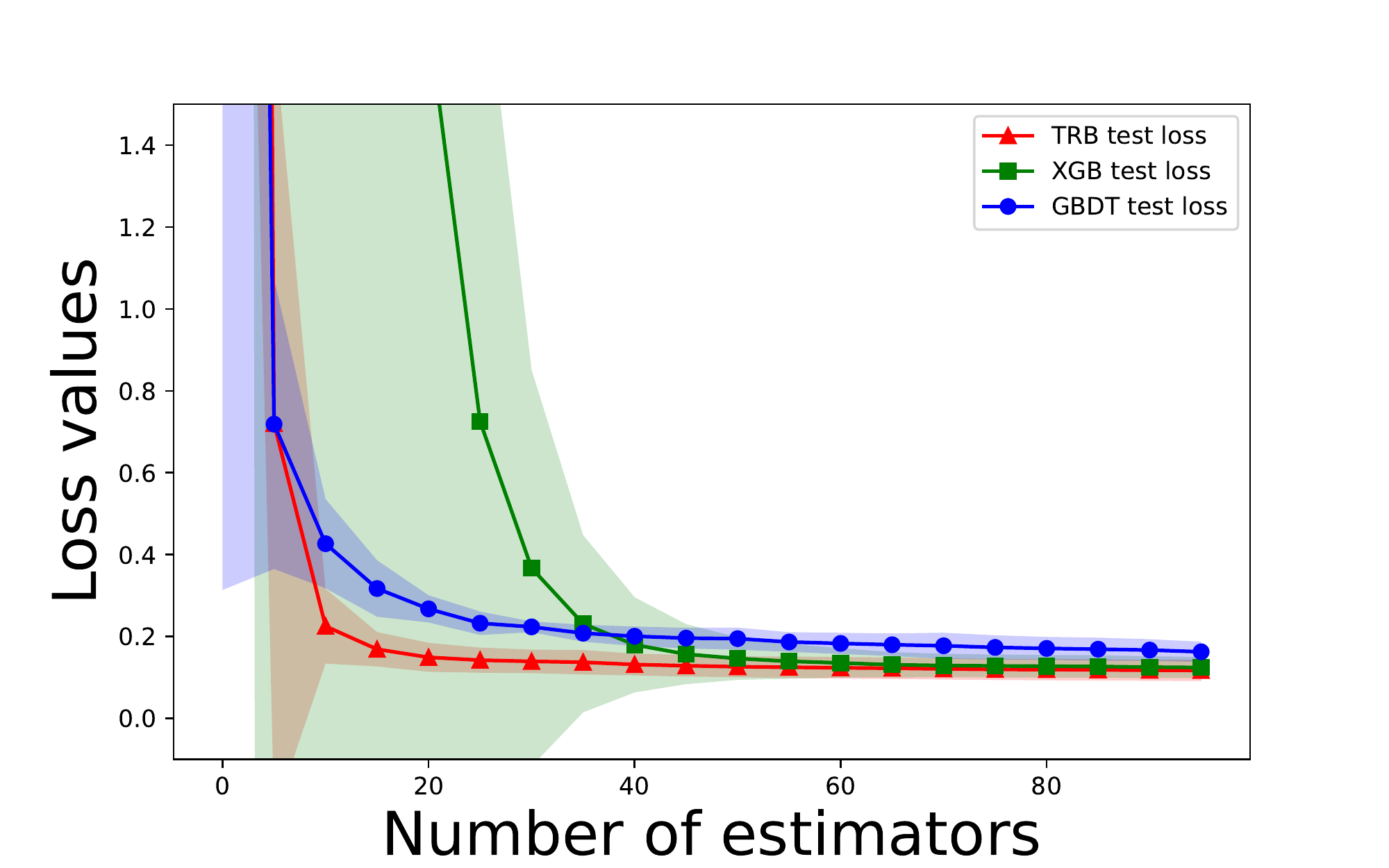}}
\subfigure[Power]{\includegraphics[width=0.22\textwidth]{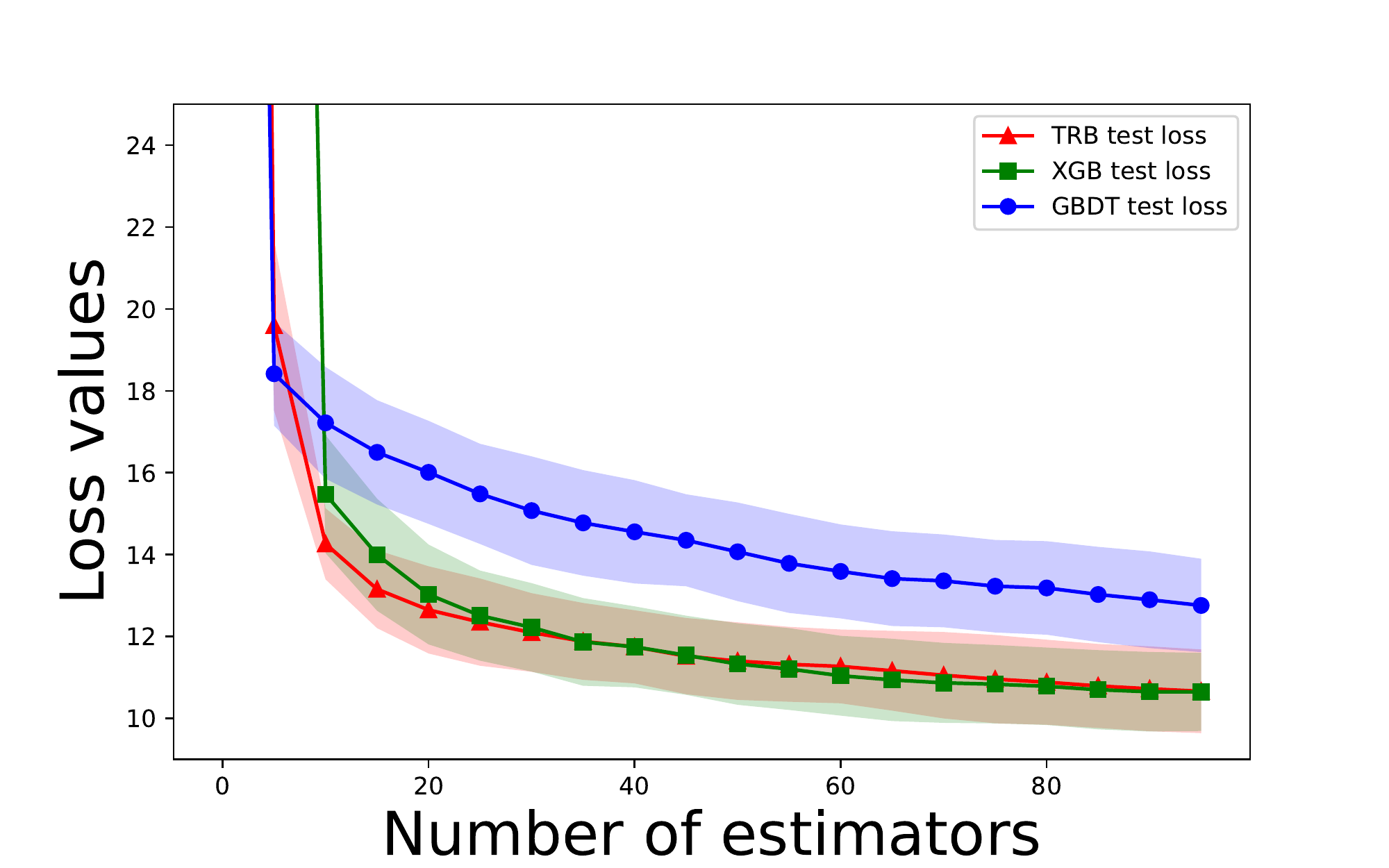}}
\subfigure[Kin8nm]{\includegraphics[width=0.22\textwidth]{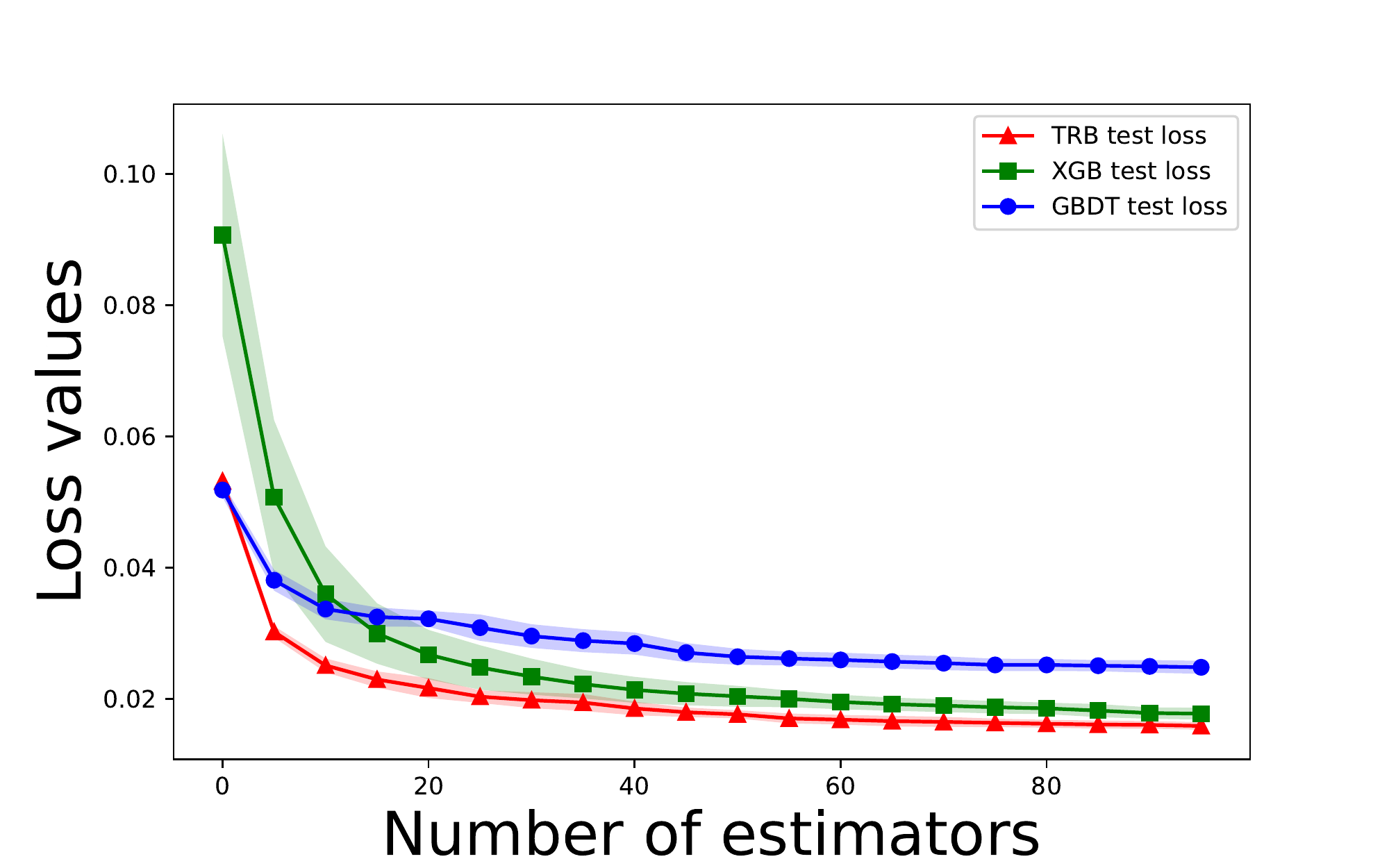}}
\subfigure[Wine]{\includegraphics[width=0.22\textwidth]{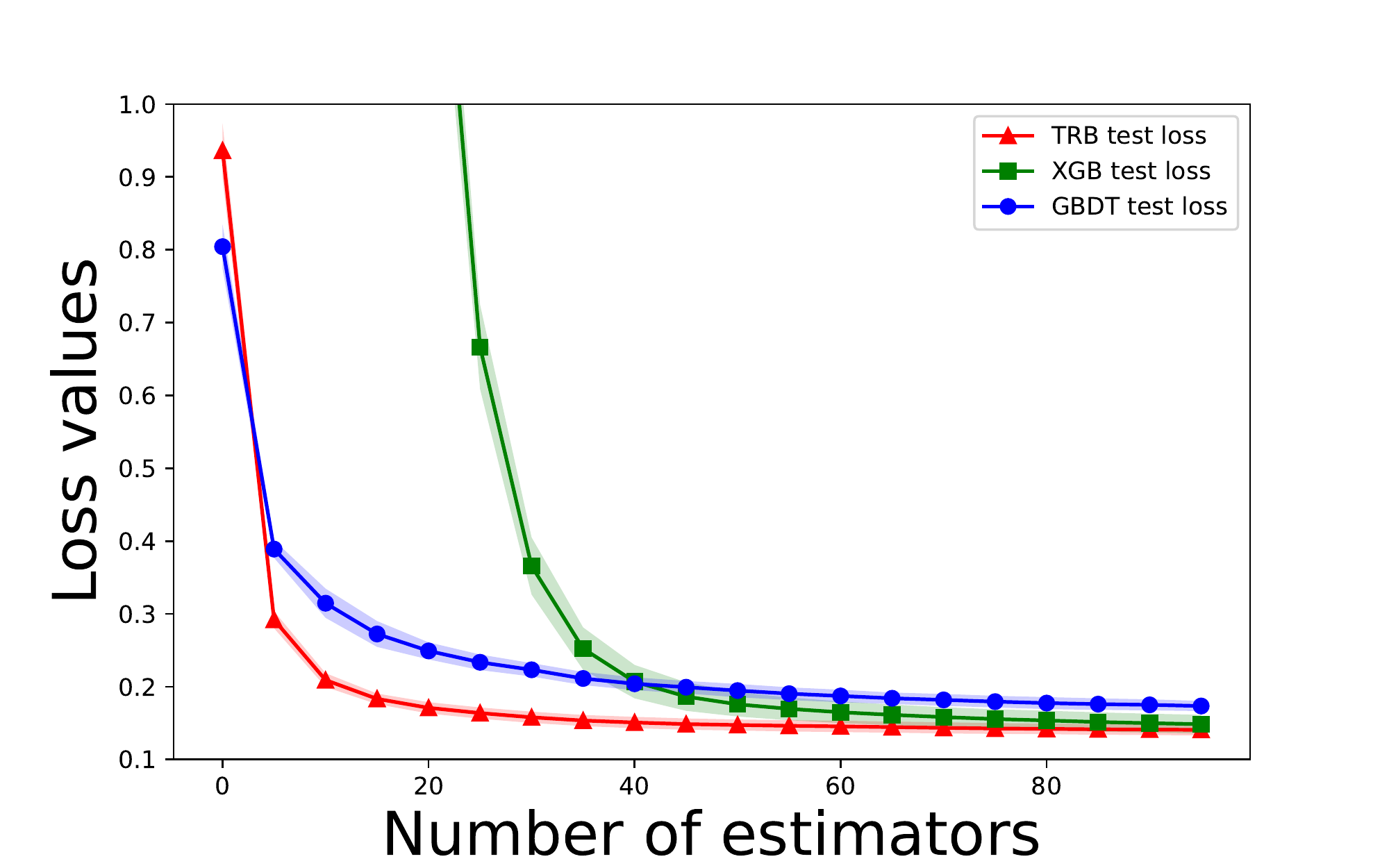}}
\subfigure[Noisy-Absolute]{\includegraphics[width=0.22\textwidth]{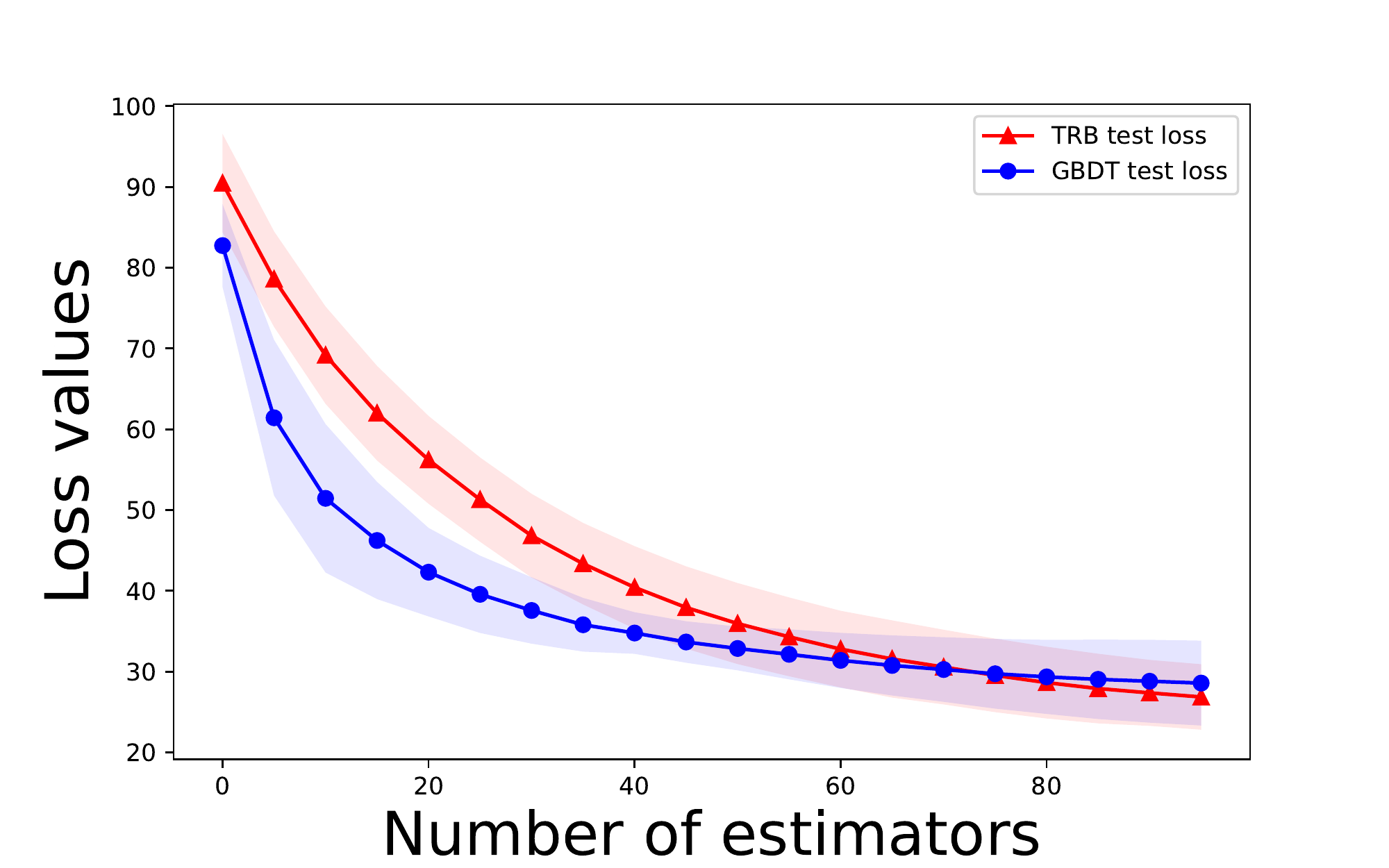}}
\subfigure[Noisy-Huber]{\includegraphics[width=0.22\textwidth]{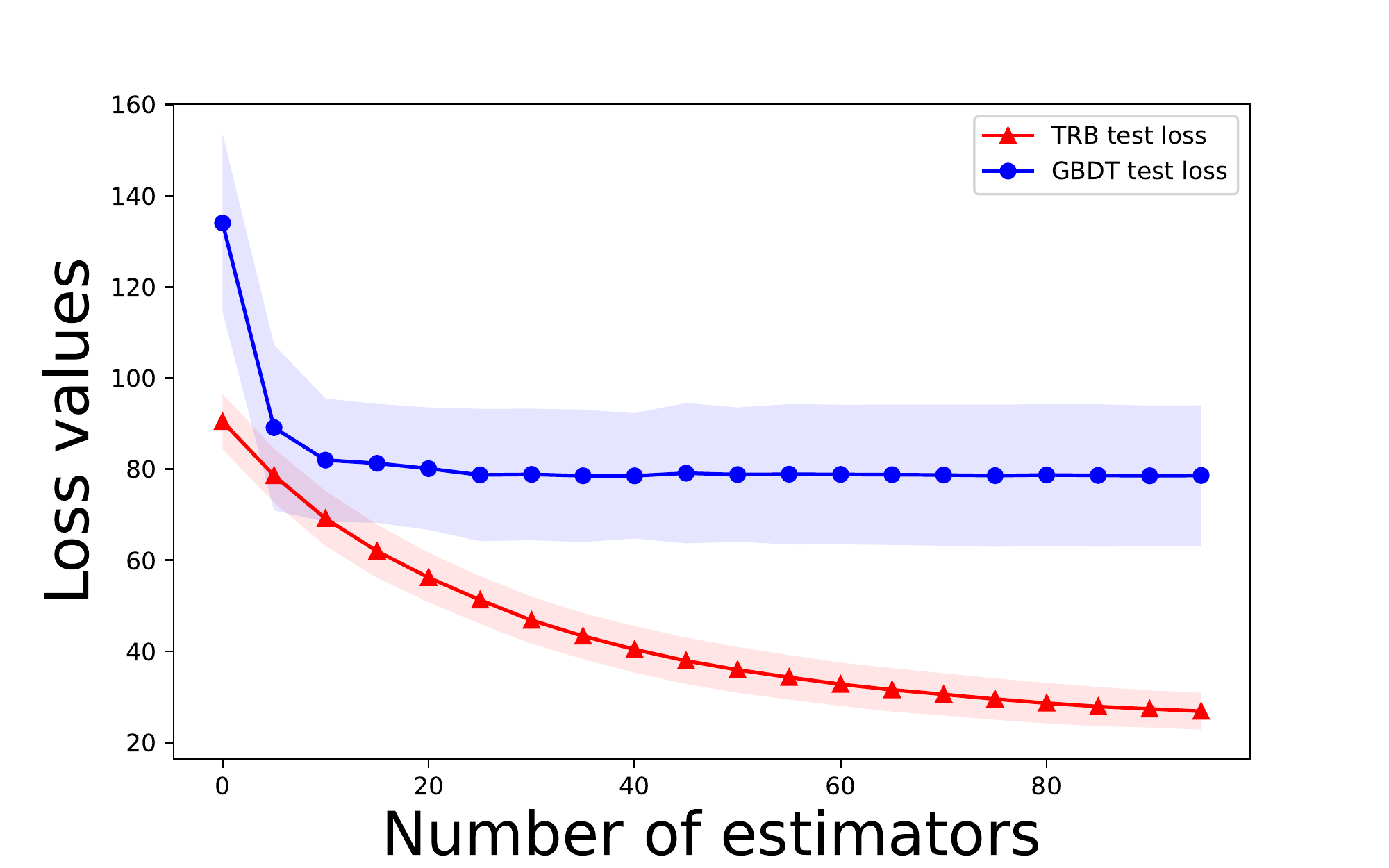}}
\caption{\textbf{Test curves of 5 trials for  GBDT, XGBoost, and TRBoost-Tree on different tasks}. Dotted lines correspond to the mean AUC/loss and the shaded areas represent the $\pm$ one standard deviation of the mean AUC/loss. Red curve: TRBoost; Green curve: XGBoost; Blue curve: GBDT.} 
\label{F.result}
\end{figure}

\subsection{Performance Comparisons}
Numerical results of different algorithms with optimal parameters are presented in Tables \ref{T.clfresult}-\ref{T.regresult}, and the test AUC and loss values of GBDT, XGBoost and TRBoost-Tree with 100 iterations are shown  in Figure \ref{F.result}.
The column named TRB-Tree denotes that the base learner is the decision tree and the approximation ratio is $\rho_1^t$, TRB-Tree(Diff) means the ratio is $\rho_2^t$, and the last two columns represent the model whose base learner is Linear Regression and Spline respectively.
To sum up, TRBoost achieves better performance and converges faster than GBDT and XGBoost in most datasets. 

Table \ref{T.clfresult} lists the results of different algorithms. It can be seen that TRBoost methods yield comparable AUC results with GBDT and XGBoost. 
At the same time,  tree-based TRBoost is more advantageous in F1-score. 
Figure \ref{F.result} (a)-(f) shows that the AUC curves obtained by TRB-Tree are similar to XGBoost and converge faster and better than those of GBDT in most cases. 

For regression problems with squared error (Figure \ref{F.result} (g)-(l)), TRBoost performs better than XGBoost and GBDT, both in terms of loss and convergence. The slow convergence of XGBoost at the beginning is because the learning rate obtained from grid search is small, and the constant Hessian also shrinks the leaf values. Although the leaf values of TRBoost also contain the Hessian, it is not multiplied by the extra learning rate, hence the convergence will be faster. We will present more details about the role of the Hessian in the next section.

From Table \ref{T.regresult}, we can see that when the loss function is not strictly convex (Noisy-Absolute and Noisy-Huber), GBDT and TRBoost are  still applicable and obtain reasonable results, while XGBoost can not. Moreover, Figure \ref{F.result} (n) shows nonzero quadratic term in loss will make TRBoost much better than GBDT both in the rate of convergence and results.

In our four boosting machines, it can be seen that the tree-based learner is still more advantageous for tabular data and both approximation ratios can obtain state-of-the-art results. Although other learners can also perform well in some cases, they are limited by the datasets. 
For example, a Linear Regression learner is fast but tends not to work well when the data distribution is complex. In contrast, Spline learners can get better results but the calculation time will be proportional to both the feature dimension and the data size. We also use other models such as SVR, Gaussian Process, and Polynomial as base learners for experiments and find that they are all restricted to the datasets. In general, if we have priors about the data, such as the approximate distribution, then a learner based on these priors will allow us to get satisfactory results with few estimators, which will greatly reduce the computational cost. Otherwise, the tree-based learner is still the first choice.

\begin{sidewaystable}
\renewcommand\arraystretch{1.5}
    \caption{Results (Mean$\pm$Std.$\%$) of different models on classification tasks.}
    \label{T.clfresult}
    \centering
    \begin{tabular}{lccccccc}
    \hline
    Dataset & Metric & GBDT & XGB & TRB-Tree & TRB-Tree(Diff) & TRB-LR &  TRB-Spline \\
    \hline
    \multirow{2}{*}{Adult} & AUC  &92.40$\pm$0.08  &\textbf{92.69$\pm$0.15}  &\textbf{92.69$\pm$0.15}  &92.65$\pm$0.15 &83.76$\pm$0.28 & 90.48$\pm$0.20\\
      & F1-score &70.77$\pm$0.16  &70.92$\pm$0.40 &\textbf{71.51$\pm$0.47}  &71.08$\pm$0.45 &31.52$\pm$1.67 & 57.91$\pm$1.20  \\
    \hline
    
    \multirow{2}{*}{German} & AUC  &76.25$\pm$3.48  &75.07$\pm$2.27 &76.99$\pm$2.66  &76.42$\pm$2.16 &\textbf{77.25$\pm$3.87}   &76.02$\pm$4.77 \\
      & F1-score  &40.20$\pm$12.13  &46.12$\pm$3.17  &\textbf{52.74$\pm$3.48}  &49.34$\pm$5.35 &41.43$\pm$22.50 & 43.94$\pm$12.34   \\
    \hline
     
    \multirow{2}{*}{Electricity} & AUC  &95.08$\pm$0.26   &\textbf{97.70$\pm$0.08}    &97.41$\pm$0.08  &97.07$\pm$0.01  &80.81$\pm$0.78   &86.26$\pm$0.42 \\
    & F1-score  &85.88$\pm$0.46  &\textbf{90.86$\pm$0.36}  &90.10$\pm$0.14  &89.43$\pm$0.36 &65.21$\pm$3.43 & 73.35$\pm$0.82  \\
    \hline 
     
    \multirow{2}{*}{Sonar}  & AUC & 92.13$\pm$2.23  &  93.93$\pm$2.49   & 94.36$\pm$1.90 & \textbf{94.59$\pm$1.99} &78.04$\pm$13.13  &75.71$\pm$7.79 \\
    & F1-score &81.35$\pm$7.53  &85.33$\pm$6.51  &84.74$\pm$5.62  &\textbf{86.77$\pm$4.19} &64.79$\pm$12.75 & 66.02$\pm$8.34   \\
    \hline
     
    \multirow{2}{*}{Credit}  &AUC & 74.91$\pm$1.50  &74.90$\pm$2.87   &74.44$\pm$1.82  &73.27$\pm$2.11 &68.24$\pm$9.60   &\textbf{76.01$\pm$4.48}  \\
    & F1-score  &83.68$\pm$1.09  &84.16$\pm$2.12  &\textbf{84.55$\pm$1.72}  &83.34$\pm$1.55 &82.63$\pm$1.50 & 83.76$\pm$1.99    \\
    \hline
     
    \multirow{2}{*}{Spam}   &AUC & 98.24$\pm$0.19  &  98.64$\pm$0.09   & \textbf{98.72$\pm$0.06} & 98.68$\pm$0.01 &95.86$\pm$0.72  &97.11$\pm$0.49\\
    & F1-score &93.36$\pm$0.56  &93.99$\pm$0.86  &94.06$\pm$0.63  &\textbf{94.19$\pm$0.59} &84.06$\pm$1.30 & 89.74$\pm$1.61   \\
    \hline
    \end{tabular}
\end{sidewaystable}

\begin{sidewaystable}
\renewcommand\arraystretch{1.6}
    \caption{Results (Mean$\pm$Std.) of different models on regression tasks.}
    \label{T.regresult}
    \centering
    \begin{tabular}{lcccccc}
    \hline
    Dataset & GBDT & XGB & TRB-Tree & TRB-Tree(Diff) & TRB-LR &  TRB-Spline \\
    \hline
    California  &0.1939$\pm$0.0097  &0.1630$\pm$0.0061 &0.1646$\pm$0.0073  &\textbf{0.1622$\pm$0.0089} &0.3876$\pm$0.0136   & 0.3115$\pm$ 0.0052\\
    
    Concrete  &18.5382$\pm$3.5001  &18.8673$\pm$5.5329   &17.5351$\pm$4.8169  &\textbf{17.4294$\pm$4.7445} &103.4283$\pm$5.9466   & 33.7457$\pm$3.6076  \\
    
    Energy  &0.1723$\pm$0.0275  &0.1237$\pm$0.0225   &\textbf{0.1228$\pm$0.0238}  &0.1273$\pm$0.0124 &9.3203$\pm$0.7809   &1.2059$\pm$0.0711  \\
    
    Power  &12.6886$\pm$1.1460  &10.6783$\pm$0.9315   &\textbf{10.6199$\pm$1.0479}  &10.6623$\pm$0.9368 &21.7999$\pm$1.1445  & 18.5266$\pm$1.1151 \\
    
    Kin8nm  &0.0251$\pm$0.0010  &0.0177$\pm$0.0009   &\textbf{0.0159$\pm$0.0006}  &0.0167$\pm$0.0004 &0.0410$\pm$0.0011   &0.0392$\pm$0.0011 \\
    
    Wine quality  &0.1735$\pm$0.0089  &0.1475$\pm$0.0130   &\textbf{0.1407$\pm$0.0077}  &0.1442$\pm$0.0082 &0.2749$\pm$0.0188    &0.2190$\pm$0.0124 \\
    
    \hline
    Noisy-Absolute  &28.6531$\pm$5.8856  &-   &\textbf{27.1661$\pm$3.3732}  &29.0376$\pm$3.8524 &86.4291$\pm$6.2855    &87.6589$\pm$6.2304 \\
    
    Noisy-Huber  &37.8729$\pm$4.1529  &-   &\textbf{26.4561$\pm$3.9762}  &31.2837$\pm$4.8995 &85.2901$\pm$5.8389    &87.3820$\pm$6.2032 \\
    \hline
    \end{tabular}
\end{sidewaystable}

\section{Theoretical Analysis}
In this section, we first prove that TRBoost has the same convergence as Newton's method when the loss is strictly convex. And when the Hessian is not positive definite, the same convergence result as the first-order method can be obtained. After that, we then discuss the impact of Hessian on our algorithm.
\subsection{Convergence analysis}
We follow the proof ideas of Sun et al. \citep{sun2014convergence} and use the same notations. 
We present the main results below and some supplemental  materials are provided in Appendix \ref{appd.theory}. For the proofs and more details of the introduced theorems, we refer interested readers to the original paper \citep{sun2014convergence}. 

\begin{theorem}[Convergence rate of TRBoost]
When the loss is the mean absolute error (MAE), TRBoost has $O(\frac{1}{T})$ rate; when the loss is MSE, TRBoost has a quadratic rate and a linear rate is achieved when the loss is Logistic loss.
\end{theorem}

\subsubsection{One-Instance Example}
Denote the loss at the current iteration by $l=l^t(y, F)$ and that at the next iteration by $l^{+} = l^{t+1}(y, F+f)$. Suppose the steps of gradient descent GBMs, Newton's GBMs, and TRBoost, are $-\nu g$, $\frac{-\nu g}{h}$, and $\frac{-g}{h+\mu}$, respectively. $\nu$ is the learning rate and is usually less than 1 to avoid overfitting.

For regression tasks with MAE loss, since the Hessian $h=0$, we have $l^{+}=l+gf$. Let $\mu=\frac{1}{v}$, then TRBoost degenerate into gradient descent and has the same convergence as the gradient descent algorithm, which is \textbf{$O(\frac{1}{T})$}.
\begin{proof}
    Since $g=1$ or -1, substituting $f=-\nu g$ to $l^{+}=l+gf$, we have
    \begin{equation}
        l^{+}=l-\nu g^2=l-\nu < l.
    \end{equation}
    Because $l$ is monotonically decreasing and $l > 0$, $\nu$ is a constant, thus there exists a constant $0<c<1 $ such that 
    \begin{equation}
        c(l^t)^2 < \nu, \forall t.
    \end{equation}
    Hence we have 
    \begin{equation}
        l^{+}=l-\nu < l-cl^2.
    \end{equation}   
    According to the Theorem \ref{thm.1/t} in Appendix, the $O(\frac{1}{T})$ convergence rate is obtained.
\end{proof}

For MSE loss, it is a quadratic function. Hence just let $\nu=1$ and $\mu=0$, then follow the standard argument in numerical optimization book \citep{nocedal1999numerical}, both Newton's method and trust-region method have the \textbf{quadratic convergence}.

We now discuss the convergence when the Hessian is not constant.
We prove that the trust-region method has a \textbf{linear convergence rate} like Newton's method with fewer constraints.
\begin{proof}
By the Mean Value Theorem, we have 
\begin{equation}
\label{e.one_mean}
    l^+ = l + g f + \frac{1}{2}h_{\xi}f^2,
\end{equation}
$\xi \in [0, f]$ and $h_{\xi}$ denotes the Hessian at $\xi$. Substituting $f=\frac{-g}{h+\mu}$ to \eqref{e.one_mean}, we have
\begin{equation}
    l^+ = l-\frac{g^2}{h+\mu} + \frac{1}{2}h_{\xi}\frac{g^2}{(h+\mu)^2}.
\end{equation}
In paper \citep{sun2014convergence}, it needs $f \ge 0$ to ensure $\frac{h_{\xi}}{h} \le 1$. But in our method, we do not need this constraint. Since the Hessian $h=4p(1-p)$ (Definition \ref{def.1}) is bounded by the predicted probability $p$ and a proper $\mu$ (for example $\mu=1$) can always make $\frac{h_{\xi}}{h+\mu} \le 1$. 

By applying Theorem \ref{thm.comp} we have
\begin{align}
    l^+ &= l-\frac{g^2}{h+\mu} + \frac{1}{2}\frac{h_{\xi}}{h+\mu}\frac{g^2}{h+\mu} \\
        &\le l-\frac{1}{2}\frac{g^2}{h+\mu} \nonumber \\
        &=  l-\frac{1}{2}\frac{h}{h+\mu}\frac{g^2}{h} \nonumber \\
        &\le l-\frac{1}{2}\frac{h}{h+\mu}l \nonumber \\
        &= (1-\frac{1}{2}\frac{h}{h+\mu})l. \nonumber 
\end{align}
For any $\mu$ that satisfies the inequality $\frac{h_{\xi}}{h+\mu} \le 1$ and the Assumption \ref{assu.bound}, $0 < (1-\frac{1}{2}\frac{h}{h+\mu}) < 1$ always holds, hence we obtain the \textbf{linear convergence} according to Theorem \ref{thm.linear}.
\end{proof}

\subsubsection{Tree model case}
We omit the proof for regression tasks with different loss functions since it is the same as the One-instance case.
\begin{proof}
Define the total loss $L(f)=\sum_{i=1}^{n}l(y_i, F_i+f_i)$. Write it in matrix form and apply the Mean Value Theorem, we have 
\begin{equation}
    L(f) = L(0) + \mathbf{g}^{\intercal}\mathbf{f} + \frac{1}{2} \mathbf{f}^{\intercal}\mathbf{H}_{\xi}\mathbf{f}, 
\end{equation}
where $\mathbf{g}=(g_1, \cdots, g_{n} )^{\intercal}$, $\mathbf{H}=diag(h_1,\cdots,h_{n})$ denote the instance wise gradient and Hessian, respectively.
In a decision tree, instances falling into the same leaf must share a common fitted value. We can thus write down $\mathbf{f}=\mathbf{V}\mathbf{s}$, where $  \mathbf{s} = (s_1, \cdots , s_j, \cdots , s_J)^{\intercal}\in \mathbb{R}^{J}$ are the fitted values on the $J$ leaves and $\mathbf{V}\in \mathbb{R}^{n\times J}$ is a projection matrix $\mathbf{V}=[\mathbf{v}_1, \cdots, \mathbf{v}_j, \cdots, \mathbf{v}_{J}]$, where $\mathbf{v}_{j, i}=1$ if the $j$th leaf contains the $i$th instance and 0 otherwise.
Substituting $\mathbf{f}=\mathbf{V}\mathbf{s}$ to $L(f)$ and reload the notation $L(\cdot)$ for $\mathbf{s}$ we have 
\begin{equation}
    L(\mathbf{s}) = L(0) + (\mathbf{g}^{\intercal}\mathbf{V})\mathbf{s} + \frac{1}{2} \mathbf{s}^{\intercal}(\mathbf{V}^{\intercal}\mathbf{H}_{\xi}\mathbf{V})\mathbf{s}.
\end{equation}

Denote $L(\mathbf{s})$ by $L^+$ and $L(0)$ by $L$ and replace $\mathbf{f}$ with the trust-region step
\begin{align}
\mathbf{f}&=-\mathbf{V}\mathbf{s} \\
          &= -\mathbf{V}(\mathbf{V}^{\intercal}\mathbf{H}\mathbf{V}+\mu \mathbf{I})^{-1} \mathbf{V}^{\intercal} \mathbf{g}, \nonumber
\end{align}
then we have
\begin{align}
L^+=&~L-(\mathbf{V}^{\intercal}\mathbf{g})^{\intercal}\hat{\mathbf{H}}^{-1}\mathbf{V}^{\intercal}\mathbf{g} \\
 & + \frac{1}{2} (\mathbf{V}^{\intercal}\mathbf{g})^{\intercal}\hat{\mathbf{H}}^{-1}(\mathbf{V}^{\intercal}\mathbf{H}_{\xi}\mathbf{V})\hat{\mathbf{H}}^{-1}\mathbf{V}^{\intercal}\mathbf{g}, \nonumber
\end{align}
where $\hat{\mathbf{H}}=\mathbf{V}^{\intercal}\mathbf{H}\mathbf{V}+\mu \mathbf{I}$, $I$ is the identity matrix.

Using the Lemma \ref{lemma.nodewise}, Lemma \ref{lemma.weak3}, and Corollary \ref{cor.bound} in the Appendix, we obtain the following inequality
\begin{align}
L^+\le &~L-(\mathbf{V}^{\intercal}\mathbf{g})^{\intercal}\hat{\mathbf{H}}^{-1}\mathbf{V}^{\intercal}\mathbf{g} + \frac{\tau}{2}(\mathbf{V}^{\intercal}\mathbf{g})^{\intercal}\hat{\mathbf{H}}^{-1}\mathbf{V}^{\intercal}\mathbf{g})~~~(Lemma~4)\\
 =&~L-(1-\frac{\tau}{2})(\mathbf{V}^{\intercal}\mathbf{g})^{\intercal}\hat{\mathbf{H}}^{-1}\mathbf{V}^{\intercal}\mathbf{g} \nonumber \\
 \le&~L-\lambda\gamma_{*}^2 (1-\frac{\tau}{2})\mathbf{g}^{\intercal}\mathbf{H}^{-1}\mathbf{g} ~~~ (Lemma ~5)\nonumber \\
 \le&~L-\lambda\gamma_{*}^2 (1-\frac{\tau}{2})L ~~~(Corollary~1) \nonumber \\
 =&~(1-\lambda\gamma_{*}^2 (1-\frac{\tau}{2}))L. \nonumber
\end{align}
Since $\tau$ and $\gamma_{*}$ are constants, $0 < \lambda \le \frac{h^{min}}{n(h^{max}+\mu)}$, a proper $\lambda$ can make $0 < \lambda\gamma_{*}^2 (1-\frac{\tau}{2}) < 1$  . According to the Theorem \ref{thm.linear} in the Appendix, the linear rate is achieved.
\end{proof}

\subsection{Discussions on Hessian}
\subsubsection{Impact of Hessian}
We further explore the impact of second-order approximation in Gradient Boosting, suppose the step of first-order GBMs, second-order GBMs, and TRBoost, is $-\nu g$, $\frac{-\nu g}{h}$, and $\frac{-g}{h+\mu}$,  respectively.

Figure \ref{F.clfloss} gives the loss curves of the binary classification tasks previously used. It can be seen that the loss curves of XGBoost and TRBoost converge faster and better than those of GBDT. There are two reasons for this, one is that the Hessian $h$, which is a function of the predicted probability $p$, is less than 1 and monotonically decreasing, hence it makes $\vert \frac{-\nu g}{h}\vert$ greater than $\vert-\nu g\vert$ and accelerates the decline of the objective function. The other is that $\vert\frac{\nu g}{h}\vert>\vert\nu g\vert$ makes the probability $p = \psi(\frac{\nu g}{h})$ predicted by the second-order method closer to 0 or 1 than the first-order method, so the corresponding objective function is smaller. However, while their losses are quite different, the AUC results of these methods are closer. This is because the AUC score does not depend on the probability, the rank of the prediction is more important.
For regression tasks, $h=2$ shrinks the predicted value $\frac{-\nu g}{h}$, hence the second-order algorithm converges slowly. In Figure \ref{F.result}(g-l), we can see that XGBoost is indeed the slowest converge because $\nu<$1 further narrows its predictions, TRBoost has similar convergence to GBDT because a suitable $\mu$ can make $\vert\frac{-g}{h+\mu}\vert \thickapprox \vert-\nu g\vert>\vert\frac{-\nu g}{h}\vert$.

Based on the above discussion, we claim that gradient determines the targets and hessian can further improve the results.
The learning rate is to reduce the size of the predicted values and find a better solution. However, combined with the expression of TRBoost, the learning rate seems unnecessary since we can always decrease the predicted value by shrinking the target value. Hence, $\frac{-g}{h+\mu}$ is a general target value that combines the advantages of both $-g$ and $\frac{-g}{h}$ and the adaptive adjustment mechanism of $\mu$ makes it more flexible.

\begin{figure}[ht]
\centering
\subfigure[Adult]{\includegraphics[width=0.22\textwidth]{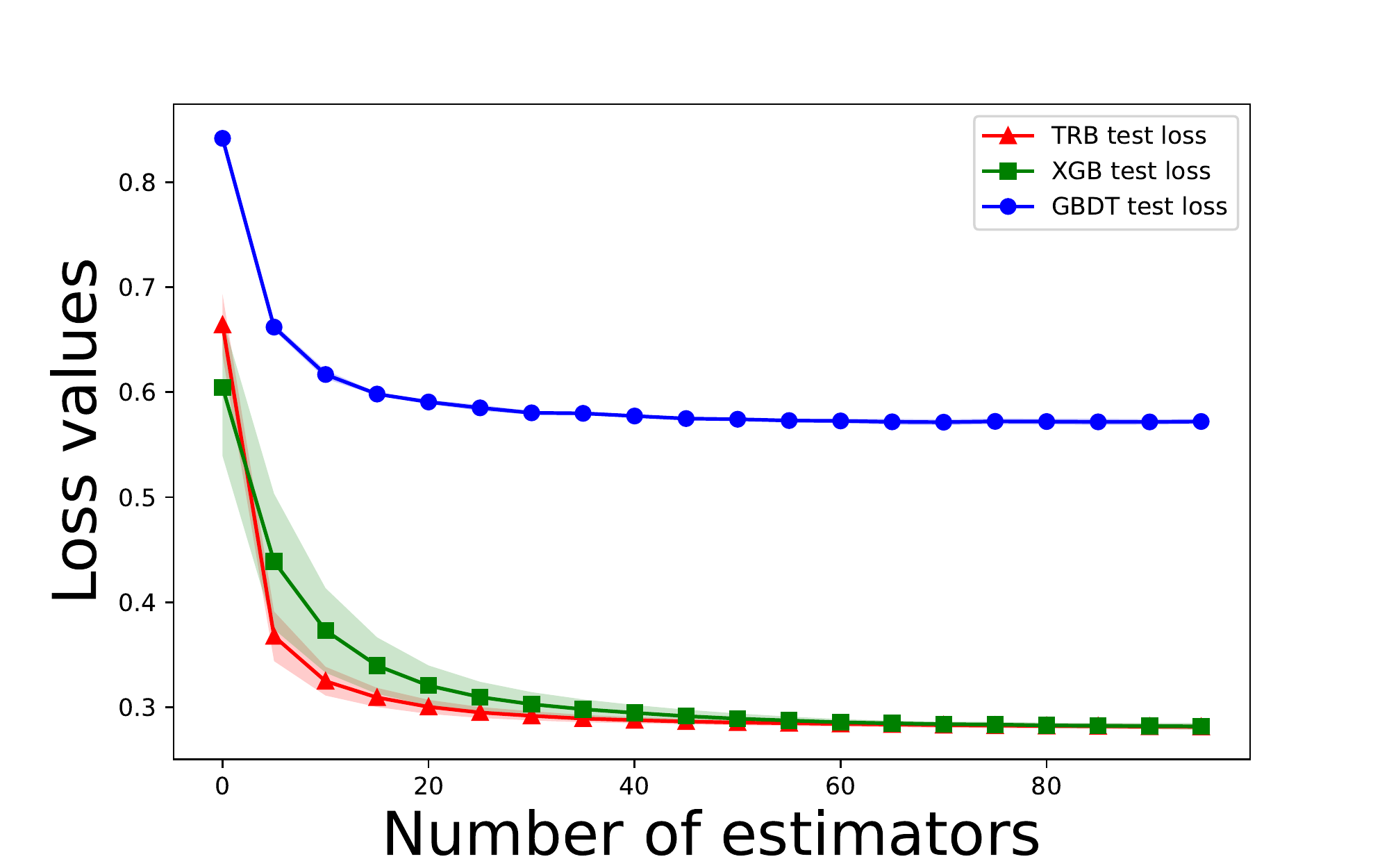}}
\subfigure[German]{\includegraphics[width=0.22\textwidth]{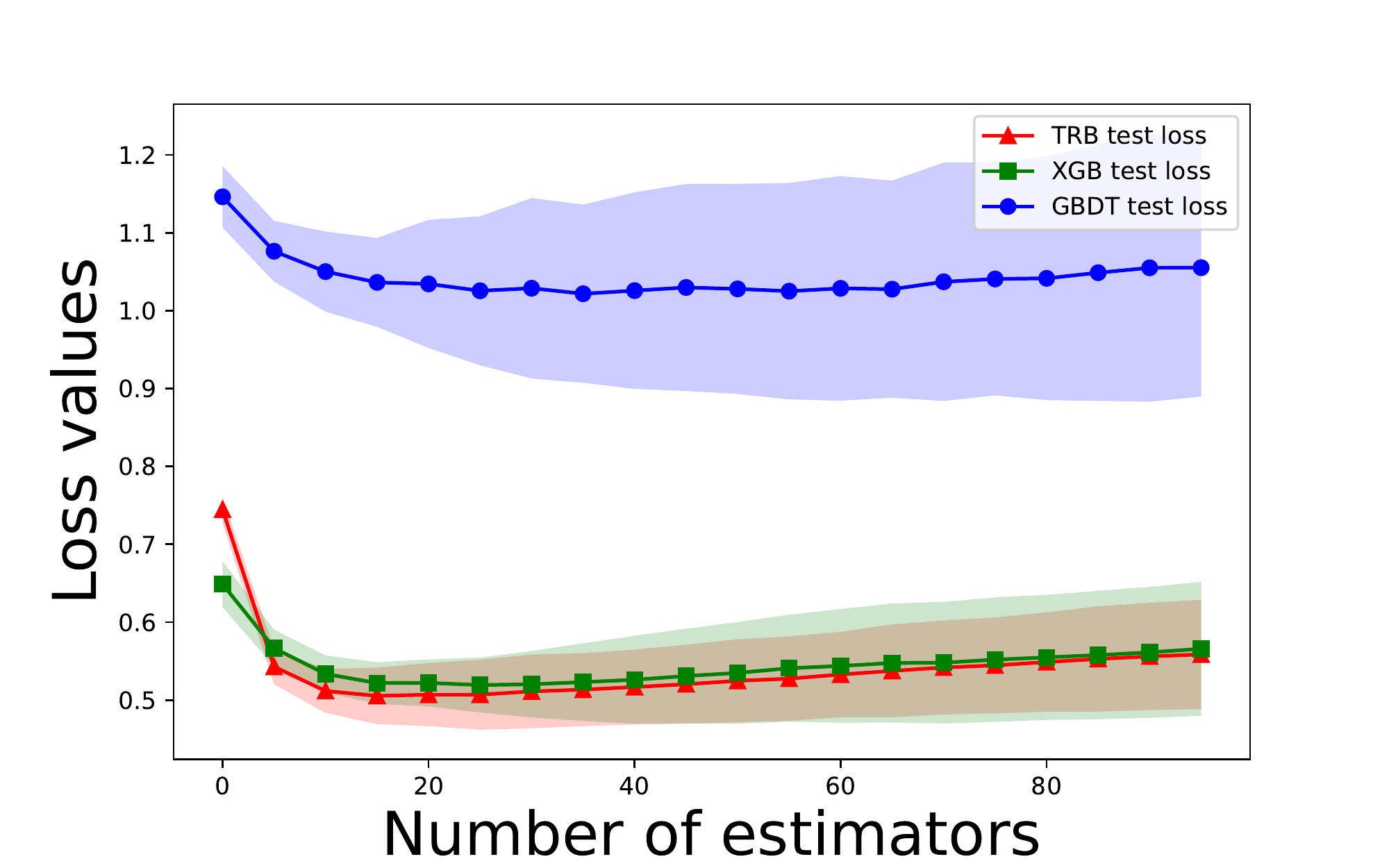}}
\subfigure[Electricity]{\includegraphics[width=0.22\textwidth]{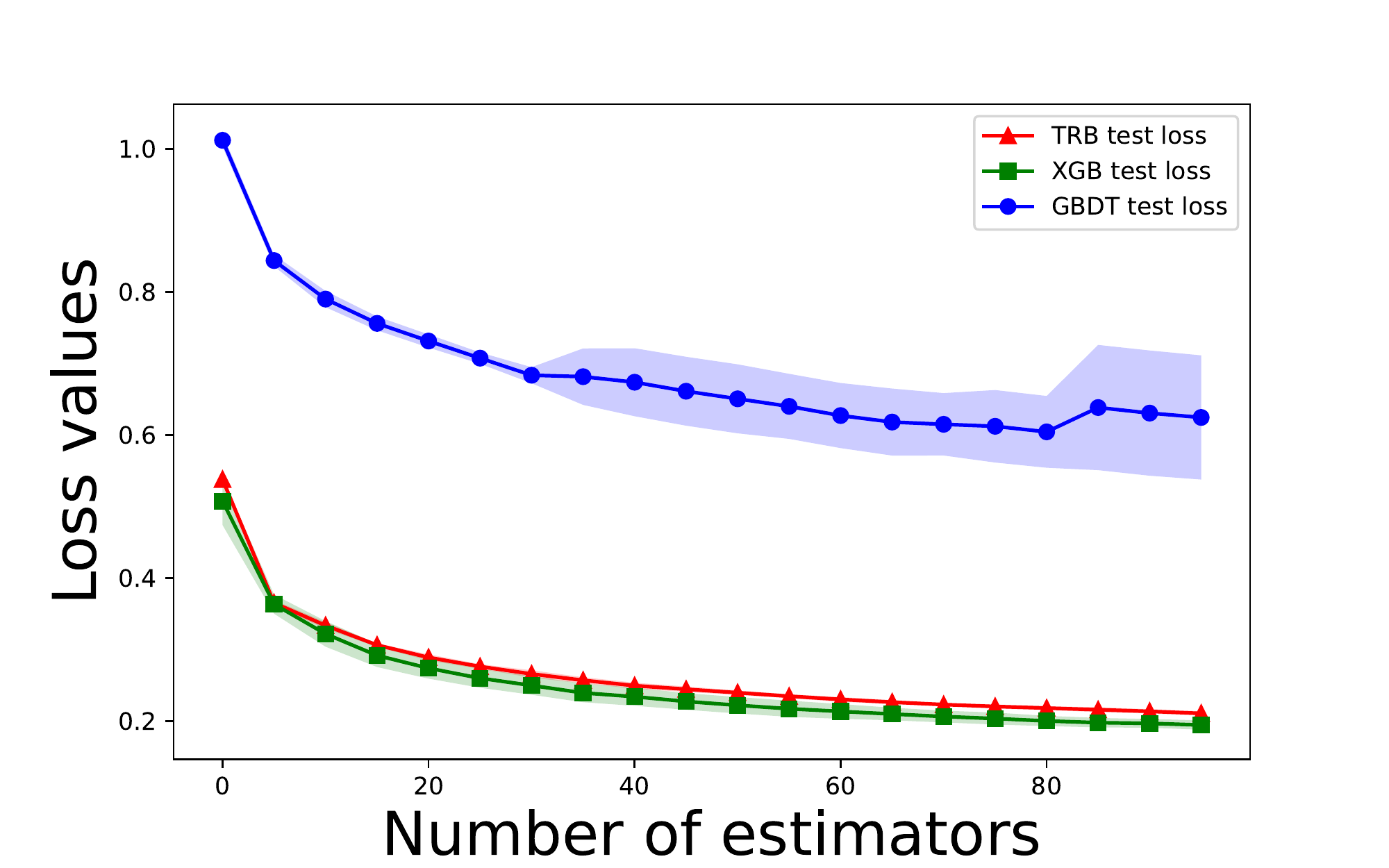}}
\subfigure[Sonar]{\includegraphics[width=0.22\textwidth]{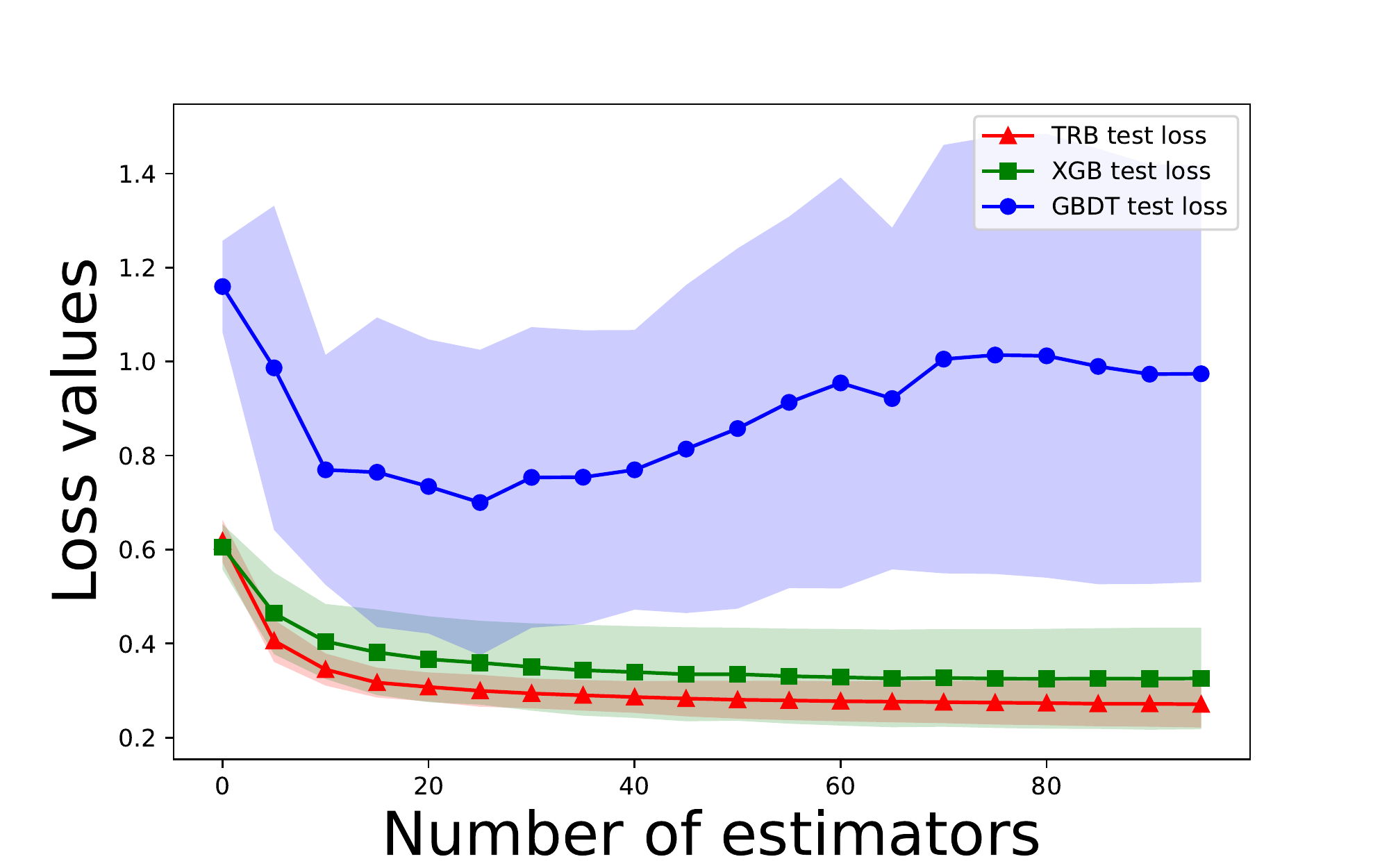}}
\subfigure[Credit]{\includegraphics[width=0.22\textwidth]{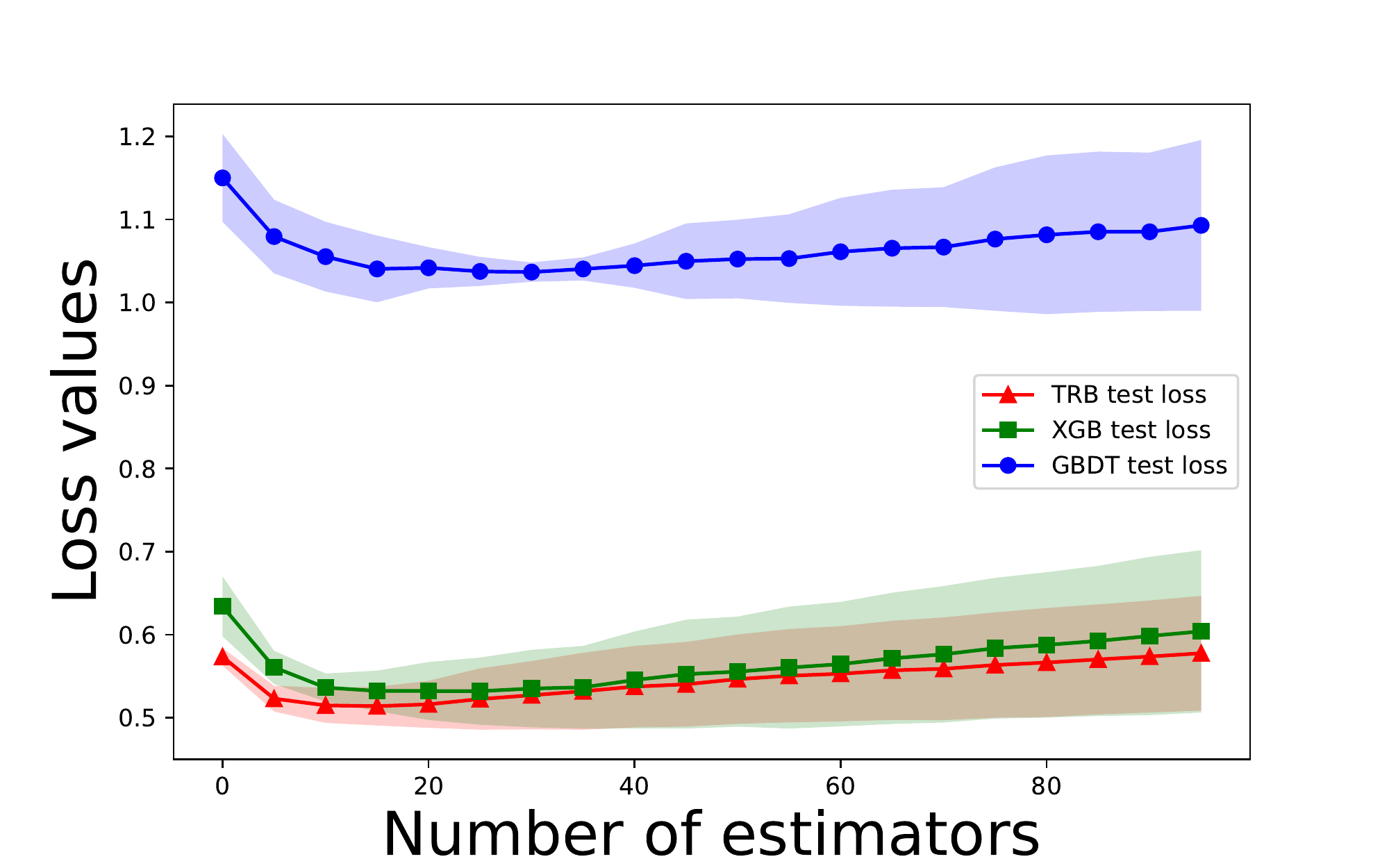}}
\subfigure[Spam]{\includegraphics[width=0.22\textwidth]{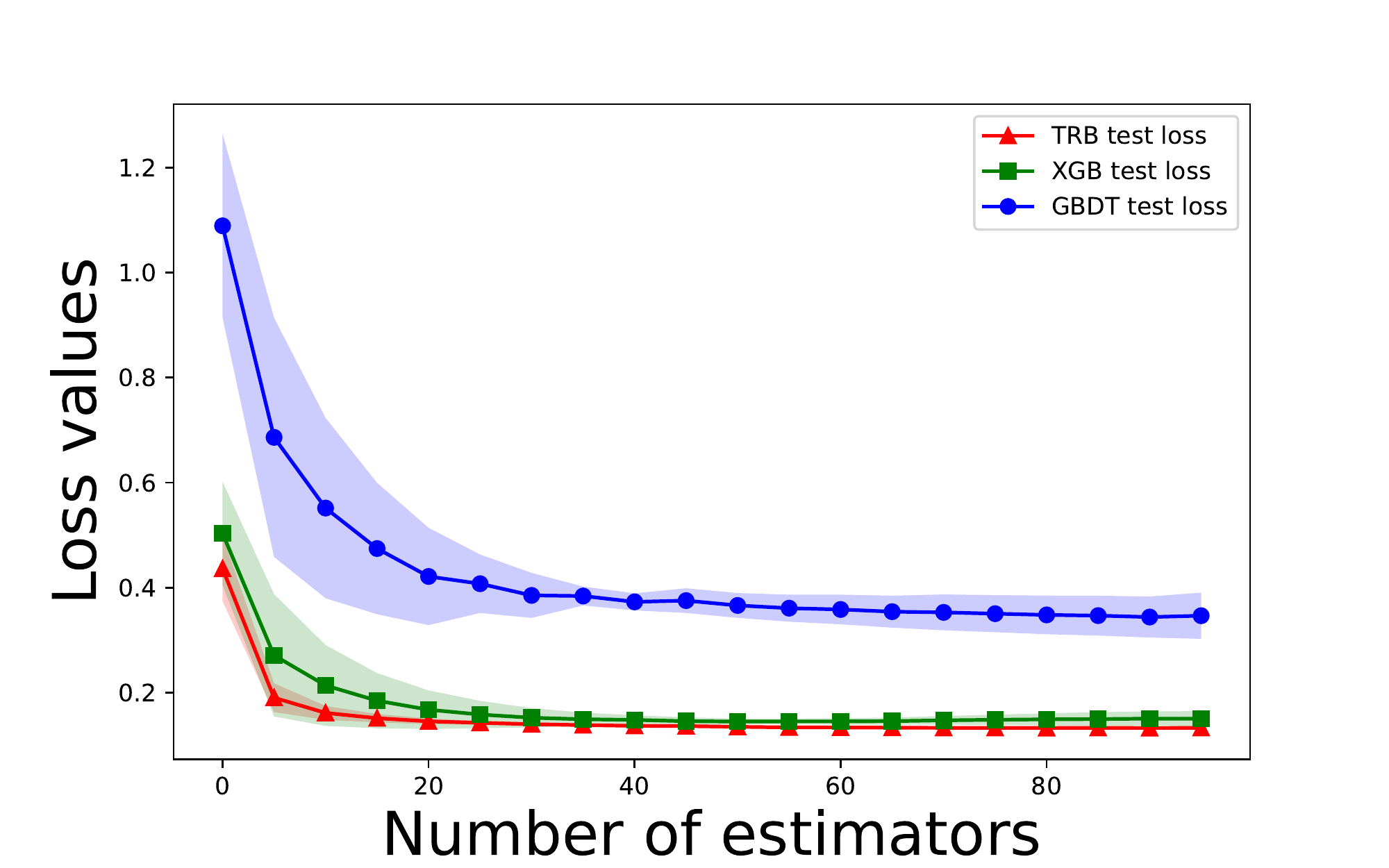}}
\caption{\textbf{Test loss curves of 5 trials for GBDT, XGBoost, and TRBoost-Tree on different classification tasks}. Dotted lines correspond to the mean loss and the shaded areas represent the $\pm$ one standard deviation of the mean loss. Red curve: TRBoost; Green curve: XGBoost; Blue curve: GBDT.} 
\label{F.clfloss}
\end{figure}

\subsubsection{Numerical Examples}
We further verify the assertion by conducting some experiments with TRBoost-Tree. We take the regression task as an example and apply the Concrete data for the test. For all results, the only difference is the choice of  $B_j^{t-1}$ and $\alpha^t$ in \eqref{e.leafvalue2}. $B_j^{t-1}=0$ means that we use the first-order approximation to the objective.

From Figure \ref{F.hessian}, we can see that when $\alpha=1$, the second-order approximation is better than the first-order method in spite of convergence and results. That is because when $B_j^{t-1}=0$, a large leaf value $C_{j}^{t}$ leads to a sharp decline in the training loss and causes overfitting. If $B_j^{t-1}>0$, it shrinks the leaf value and plays a regular role.
When $\alpha=10$, the shrinkage effect of $B_j^{t-1}$ disappears hence both approximations do nice jobs, and the results are better than those with $\alpha=1$. 
Combining the above discussions with $\lim_{C\to 0} \frac{\frac{1}{2}BC^2+GC}{GC}=1$, we conclude that the second-order approximation is rewarding only when $\alpha$ is small, i.e. the radius of trust region is large.

In numerical optimization, small steps tend to find better points but require more iterations. On the contrary, large steps can make the objective converge faster, but they may miss the optimal point. 
Hence in practice, we suggest using big $\alpha$ which means small steps to find a better model. At this time, the performance of algorithms does not rely much on the order of expansion of the objective. This also proves the generality of our method.

\begin{figure}[ht]
\centering
\subfigure{\includegraphics[scale=0.35]{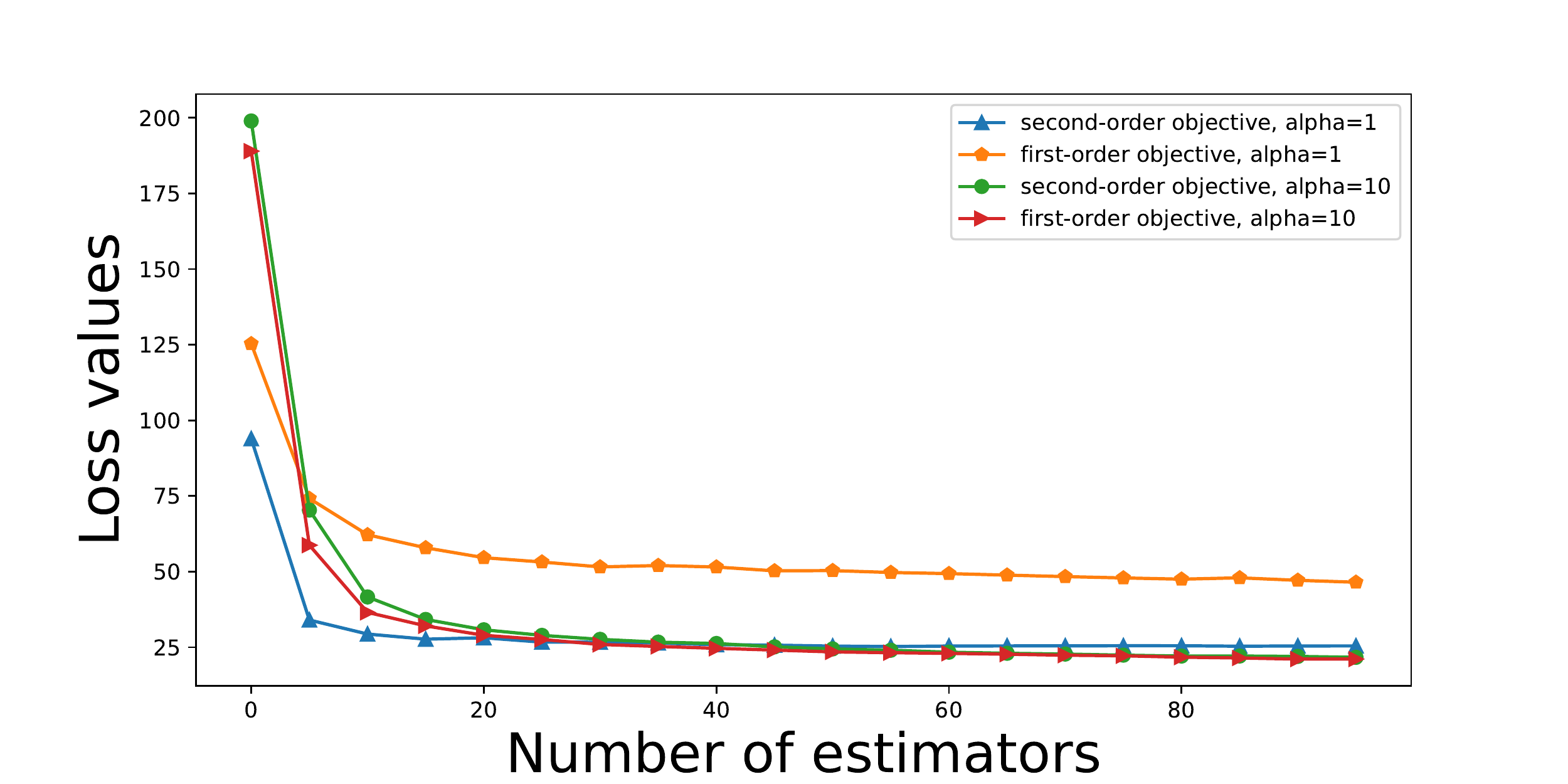}}
\caption{\textbf{Loss curves of Taylor expansion of different orders}. The orange line represents the first-order approximation result of $\alpha=1$, the blue line represents the second-order approximation result of $\alpha=1$, the red line represents the first-order approximation result of $\alpha=10$, and the green line represents the second-order approximation result of $\alpha=10$. } 
\label{F.hessian}
\end{figure}

\section{Conclusions and Future Work}
In this paper, we propose TRBoost, a generic gradient boosting machine that utilizes the Trust-region method. Our algorithm is applicable to any learner and offers the combined benefits of both first-order and second-order algorithms. Additionally, TRBoost is capable of automatically controlling the number of learners and the fitted target.
We demonstrate theoretically and numerically that our algorithm has faster convergence when dealing with strictly convex losses, outperforming the first-order algorithm and delivering comparable performance to the second-order algorithm. And it can still produce satisfactory results even when the Hessian is not positive definite.
Moreover, we conducted some analyses and experiments, which demonstrate that our approach is not affected by the degree of expansion of the loss functions. Overall, by combining these findings, TRBoost offers a more general and higher-performing solution for general supervised learning challenges.

There are many avenues for future work. First, the current version of TRBoost is implemented in a simple way. The performance can be improved by employing more advanced technology \citep{verbraeken2020survey} and better learners \citep{chen2016xgboost,ke2017lightgbm,prokhorenkova2018catboost}. Besides, since the loss functions commonly used are convex, this property may be applied to design models that converge faster while maintaining accuracy. In this case, few estimators are needed and more powerful models such as neural networks can be employed as the base learners, which is meaningful but challenging.

\begin{appendices}

\section{Related Theory}
\label{appd.theory}

\subsection{Existing theoretical results}
In this subsection, we introduce some related results proposed in the paper \citep{sun2014convergence}.
\begin{definition}[Logistic loss, gradient, Hessian]
\label{def.1}
The instance wise loss $l(\cdot, \cdot)$ is defined as
\begin{equation}
    l(y, F) = ylog(\frac{1}{p})+(1-y)log(\frac{1}{1-p}),
\end{equation}
where $y \in \{0, 1\}$ and the probability estimate $p \in [0, 1]$ is computed by a  sigmoid-like function on $F$
\begin{equation}
    p = \psi(F)=\frac{e^F}{e^F+e^{-F}}.
\end{equation}
And the gradient and Hessian of $l(\cdot, F)$ are given as 
\begin{equation}
    g(F) = 2(p-y), h(F)=4p(1-p).
\end{equation}
\end{definition}

\begin{assumption}[Clapping on $p_i$]
\label{assu.clap}
In order to avoid the numerical stability problems, we do a clapping on the output probability $p_i, i=1,\cdots,n$:
$$ p_i=\left\{
\begin{aligned}
&1-\rho, (y_i=0)\wedge(p_i > 1-\rho) \\
&\rho, (y_i=1)\wedge(p_i < \rho)\\
&p_i, otherwise
\end{aligned}
\right.
$$
for some small constant $\rho$.
\end{assumption}

\begin{theorem}[Recurrence to O($\frac{1}{T}$) rate]
\label{thm.1/t}
If a sequence $\epsilon_0 \ge \cdots \ge \epsilon_{t-1} \ge \epsilon_{t} \ge \cdots \ge \epsilon_{T} > 0$ has the recurrence relation $\epsilon_t \le \epsilon_{t-1}-c\epsilon_{t-1}^2$ for a small constant $c>0$, then we have $\frac{1}{T}$ convergence rate: $\epsilon_T \le \frac{\epsilon_0}{1+c\epsilon_0T}$.
\end{theorem}

\begin{theorem}[Recurrence to linear rate]
\label{thm.linear}
If a sequence $\epsilon_0 \ge \cdots \ge \epsilon_{t-1} \ge \epsilon_{t} \ge \cdots \ge \epsilon_{T} > 0$ has the recurrence relation $\epsilon_t \le c\epsilon_{t-1}$ for a small constant $0<c<1$, then we have linear convergence rate: $\epsilon_T \le \epsilon_0 c^T$.
\end{theorem}

\begin{theorem}[Bounded Newton Step]
\label{thm.newstep}
With the value clapping on $p_i, i=1,\cdots,n$, the Newton step $-\frac{\Bar{g}_j}{\Bar{h}_j}$ is bounded such that $|\Bar{g}_j /\Bar{h}_j| \le 1/(2\rho)$. Here $\Bar{g}_j=\sum_{i\in I_j}g_i$ and $\Bar{h}_j=\sum_{i\in I_j}h_i$ denote the sum of gradient and Hessian in $j$th tree node.
\end{theorem}

\begin{theorem}[Convergence rate]
For GBMs, it has $O(\frac{1}{T})$ rate when using gradient descent, while a linear rate is achieved when using Newton descent.
\end{theorem}

\begin{theorem}[Comparison]
\label{thm.comp}
Let $g$, $h$, and $l$ be the shorthand for gradient, Hessian, and loss, respectively. Then $\forall p$ (and thus $\forall F$), the inequality $\frac{g^2}{h}\ge l$ always holds.
\end{theorem} 

\begin{corollary}
\label{cor.bound}
Connect Newton's objective reduction and loss reduction using Theorem \ref{thm.comp}, we have
\begin{equation}
    \mathbf{g}^{\intercal}\mathbf{H}^{-1}\mathbf{g} \ge L.
\end{equation}
\end{corollary}

\begin{lemma}[Weak Learnability \uppercase\expandafter{\romannumeral1}]
\label{lemma.weak1}
For a set of $n$ instances with labels $\{0, 1\}^n$ and non-negative weights $\{w_1, \cdots, w_n\}$, there exists a $J$-leaf classification tree (i.e., outputting $\{0, 1\}$  at each leaf) such that the weighted error rate at each leaf is strictly less than $\frac{1}{2}$ by at least $\delta > 0$. 
\end{lemma}

\begin{lemma}[Weak Learnability \uppercase\expandafter{\romannumeral2}]
\label{lemma.weak2}
Let $\mathbf{g}=(g_1, \cdots, g_n )^{\intercal}$, $\mathbf{H}=diag(h_1,\cdots,h_n)$. If the weak Learnability assumption \ref{lemma.weak1} holds, then there exists a $J$-leaf regression tree whose projection matrix $\mathbf{V}\in \mathbb{R}^{n\times J}$ satisfies 
\begin{equation}
(\mathbf{V}^{\intercal}\mathbf{g})^{\intercal}(\mathbf{V}^{\intercal}\mathbf{H}\mathbf{V})(\mathbf{V}^{\intercal}\mathbf{g})  \ge 
\gamma_{*}^2 \mathbf{g}^{\intercal}\mathbf{H}^{-1}\mathbf{g}, 
\end{equation}
where $\gamma_{*}^2 = \frac{4\delta^2\rho}{n}$.
\end{lemma}

\begin{lemma}[Change of Hessian]
\label{lemma.change}
For the current $F \in \mathbb{R}$ and a step $f \in \mathbb{R}$, the change of the Hessian can be bounded in terms of step size $|f|$:
\begin{equation}
    \frac{h(F+f)}{h(f)} \le e^{2|f|}.
\end{equation}
\end{lemma}

\subsection{New theoretical results proposed in this paper}
\begin{assumption}[Boundedness of $\mu$]
\label{assu.bound}
In the Algorithm \ref{alg.trgbm}, we can see that $\mu^{t}$ is monotonically increasing, and it will become infinite as the number of iterations increase, which makes the step $\mathbf{p}_k$ go to 0. In order to avoid this situation, in practice we assume $\mu^t$ is bounded, i.e., $\mu^t \in [\mu^0, \mu^{max}]$, $\forall t$.
\end{assumption}

\begin{theorem}[Boundedness of Hessian]
\label{thm.truststep}
Since $\mu$ and Hessian $h$ are bounded, there exists a constant $\lambda$ that satisfies 
\begin{equation}
    \frac{1}{\sum_{j\in I_j}(h_j+\mu)} \ge \lambda\frac{1}{\sum_{j\in I_j}h_j}, \forall j=1,\cdots,J.
\end{equation}
\begin{proof}
Let $h^{max}$ and $h^{min}$ be the maximum and minimum values of the Hessian respectively.
The minimum value on the left side of the inequality is $\frac{1}{n(h^{max}+\mu)}$, and the maximum value on the right side is $\frac{1}{h^{min}}$. Just let 
\begin{equation}
    \lambda \le \frac{h^{min}}{n(h^{max}+\mu)},
\end{equation}
then the inequality always holds.
\end{proof}
\end{theorem}

\begin{lemma}[Node wise Hessian]
\label{lemma.nodewise}
Based on Assumption \ref{assu.clap}, we have 
\begin{equation}
    (\mathbf{V}^{\intercal}\mathbf{H}_{\xi}\mathbf{V})\hat{\mathbf{H}}^{-1} \le \tau I.
\end{equation}
\end{lemma}

\begin{proof}
For the $j$th leaf $(j = 1, \cdots, J)$, the node wise Hessian before update is
\begin{equation}
    \Bar{h}_j=\sum_{i \in I_j}(h(F_i)+\mu)
\end{equation}
while that at the mean value is 
\begin{equation}
    \Bar{h}_{\xi, j}=\sum_{i=I_j}h(F_i+\eta_j),
\end{equation}
where $\eta_j$ lies between 0 and the trust-region step $\frac{\sum_{i=I_j}g(F_i)}{\sum_{i \in I_j}(h(F_i)+\mu)}$.\\
Applying Lemma \ref{lemma.change}, we have
\begin{align}
    \frac{\Bar{h}_{\xi, j}}{\Bar{h}_j} &=~ 
    \frac{\sum_{i=I_j}h(F_i+\eta_j)}{\sum_{i \in I_j}(h(F_i)+\mu)} \\
    &\le~ \frac{\sum_{i=I_j}h(F_i+\eta_j)}{\sum_{i \in I_j}h(F_i)} \nonumber \\
    &\le~ \frac{\sum_{i=I_j}h(F_i)\cdot e^{2\eta_j}}{\sum_{i \in I_j}h(F_i)} \nonumber \\
    &=~ \exp^{2\eta_j}. \nonumber
\end{align}
According to Theorem \ref{thm.newstep}, we have
\begin{equation}
    |\eta_j| \le \frac{\sum_{i\in I_j}g(F_i)}{\sum_{i \in I_j}(h(F_i)+\mu)} \le \frac{\sum_{i\in I_j}g(F_i)}{\sum_{i \in I_j}h(F_i} \le \frac{1}{2\rho}.
\end{equation}
Rewriting the matrix form, we have
\begin{equation}
    (\mathbf{V}^{\intercal}\mathbf{H}_{\xi}\mathbf{V})\hat{\mathbf{H}}^{-1} \le \tau I,
\end{equation}
where $\tau = e^{\frac{1}{2\rho}}$.
\end{proof}

\begin{lemma}[Weak Learnability \uppercase\expandafter{\romannumeral3}]
\label{lemma.weak3}
Let $\mathbf{g}=(g_1, \cdots, g_n )^{\intercal}$, $\mathbf{H}=diag(h_1,\cdots,h_n)$, $\hat{\mathbf{H}}=\mathbf{V}^{\intercal}\mathbf{H}\mathbf{V}+\mu \mathbf{I}$. If the weak Learnability assumption \ref{lemma.weak2} holds, then there exists a $J$-leaf regression tree whose projection matrix $\mathbf{V}\in \mathbb{R}^{n\times J}$ satisfies 
\begin{equation}
(\mathbf{V}^{\intercal}\mathbf{g})^{\intercal}\hat{\mathbf{H}}^{-1}(\mathbf{V}^{\intercal}\mathbf{g})  \ge 
\gamma \mathbf{g}^{\intercal}\mathbf{H}^{-1}\mathbf{g} 
\end{equation}
for some constant $\gamma > 0$.
\end{lemma}

\begin{proof}
For each leaf $j$, according to Lemma \ref{lemma.weak2},  we have
\begin{equation}
    \frac{(\sum_{i\in I_j}g_i)^2}{\sum_{i \in I_j}h_i} \ge \gamma_{*}^2 \sum_{i=I_j}\frac{g_i^2}{h_i}
\end{equation}
By Theorem \ref{thm.truststep}, for some $\lambda$, we obtain 
\begin{equation}
    \frac{(\sum_{i\in I_j}g_i)^2}{\sum_{i \in I_j}(h_i+\mu)} \ge \lambda\frac{(\sum_{i\in I_j}g_i)^2}{\sum_{i \in I_j}h_i} \ge \lambda\gamma_{*}^2 \sum_{i=I_j}\frac{g_i^2}{h_i}.
\end{equation}
Rewriting in the matrix form, we get the final result
\begin{equation}
(\mathbf{V}^{\intercal}\mathbf{g})^{\intercal}\hat{\mathbf{H}}^{-1}(\mathbf{V}^{\intercal}\mathbf{g})  \ge 
\lambda\gamma_{*}^2 \mathbf{g}^{\intercal}\mathbf{H}^{-1}\mathbf{g}.
\end{equation}
\end{proof}


    



\end{appendices}


\bibliography{references}

\end{document}